\documentclass{article}

\usepackage[final]{neurips_2024}

\usepackage{graphicx,wrapfig,lipsum}

\usepackage[utf8]{inputenc} 
\usepackage[T1]{fontenc}    
\usepackage{hyperref}       
\usepackage{url}            
\usepackage{booktabs}       
\usepackage{amsfonts}       
\usepackage{nicefrac}       
\usepackage{microtype}      
\usepackage{xcolor}         

\newcommand{\eg}{e.\,g., }
\newcommand{\ie}{i.\,e., }
\usepackage{amsmath}
\usepackage{amssymb}
\usepackage{mathtools}
\usepackage{amsthm}
\usepackage{booktabs}
\usepackage{multirow}
\usepackage{comment}

\theoremstyle{plain}

\newtheorem*{theorem*}{Theorem}
\newtheorem*{lemma*}{Lemma}

\newtheorem{theorem}{Theorem}[]

\newtheorem{lemma}[theorem]{Lemma}

\theoremstyle{definition}
\newtheorem{definition}[theorem]{Definition}

\theoremstyle{remark}

\usepackage[colorinlistoftodos,bordercolor=orange,backgroundcolor=orange!20,linecolor=orange,textsize=scriptsize]{todonotes}

\title{If You Want to Be Robust, Be Wary of Initialization}

\author{%
  Sofiane Ennadir\thanks{Corresponding Author: \texttt{ennadir@kth.se}} \\
  KTH \\
  Stockholm, Sweden\\
  \And
  Johannes F. Lutzeyer\\
  LIX, Ecole Polytechnique\\
  IP Paris, France \\
  \AND
  Michalis Vazirgiannis \\
  KTH \& Ecole Polytechnique\\
  Stockholm, Sweden \\
  \And
  El Houcine Bergou\\
  UM6P \\
  Benguerir, Morocco \\
}

\begin{document}

\maketitle

\begin{abstract}
Graph Neural Networks (GNNs) have demonstrated remarkable performance across a spectrum of graph-related tasks, however concerns persist regarding their vulnerability to adversarial perturbations. While prevailing defense strategies focus primarily on pre-processing techniques and adaptive message-passing schemes, this study delves into an under-explored dimension: the impact of weight initialization and associated hyper-parameters, such as training epochs, on a model’s robustness.
We introduce a theoretical framework bridging the connection between initialization strategies and a network's resilience to adversarial perturbations. Our analysis reveals a direct relationship between initial weights, number of training epochs and the model’s vulnerability, offering new insights into adversarial robustness beyond conventional defense mechanisms. While our primary focus is on GNNs, we extend our theoretical framework, providing a general upper-bound applicable to Deep Neural Networks.
Extensive experiments, spanning diverse models and real-world datasets subjected to various adversarial attacks, validate our findings. We illustrate that selecting appropriate initialization not only ensures performance on clean datasets but also enhances model robustness against adversarial perturbations, with observed gaps of up to 50\% compared to alternative initialization approaches.
\end{abstract}

\section{Introduction}

Neural networks have demonstrated remarkable prowess across various domains, ranging from computer vision~\cite{dosovitskiy2020image} to natural language processing~\cite{vaswani2017attention}, proving their ability to model and extract complex insights from real-world datasets. Recently, Graph Neural Networks (GNNs) \cite{Kipf:2017tc, xu2019powerful, Velickovic:2018we} have emerged as a powerful extension of neural networks specifically tailored to tackle graph-structured data. These models have led to rapid progress in solving tasks such as node and graph classification where their application have spanned from drug design~\cite{kearnes2016molecular}, protein resistance analysis~\cite{qabel2023arg}, session-based recommendations~\cite{wu2019session} to tabular data~\cite{IGNNet}. Concurrently with their success, deep learning architectures have been shown to be unstable when subject to adversarial perturbations \cite{goodfellow2015}, resulting in unreliable predictions, consequently questioning these models' applicability in critical domains. While most adversarial robustness studies focus on the domain of computer vision, recent work \cite{gunnemann2022graph} studying the robustness of GNNs has emerged. Given their rich nature, graphs allow different attack schemes, where the attacker can either choose to edit the graph structure (by adding/deleting edges) or edit the node/edge features. In parallel, recent studies have been devoted to studying approaches to defend against these attacks and enhance GNN robustness, such as input pre-processing techniques~\cite{gnn_jaccard}, low-rank approximation \cite{gnnsvd}, edge-pruning \cite{gnn_guard} or adapting the message-passing schemes \cite{abbahaddou2024bounding}.

The majority of available defense studies focus on understanding the inner dynamics of GNNs to pinpoint and mitigate adversarial vulnerabilities. While analyzing the message-passing mechanism and implementing input pre-processing techniques remains a viable direction, comprehensive understanding necessitates exploration beyond traditional avenues. In this sense, investigating factors such as weight initialization strategies and the impact of other hyperparameters, notably those associated with optimization mechanisms, can offer new insights and perspectives on achieving GNN global robustness. Hyperparameter choices and tuning play a critical role in striking a balance between learning the underlying signals in the data and preventing overfitting to ensure the model's generalization. Hence, existing studies on initialization mainly evolve around understanding its effect on the model's convergence, stability and performance \cite{xiao2018dynamical, pennington2018emergence}. In contrast, our work primarily focuses on examining the effect of initialization on a model's underlying adversarial robustness,  representing to the best of our knowledge the first exploration of its kind. Our main objective is to provide a theoretical understanding of the link between weight initialization and other dynamics such as the number of training steps and the resulting model's robustness.
With this perspective in mind, we start by formalizing robustness in the context of GNNs when subjected to structural and node feature-based adversarial attacks. Subsequently, we derive an upper bound that connects the model's robustness to the weight initialization strategies. Specifically, we illustrate that this bound depends on the initial weight norms and the number of training epochs. Finally, we validate our theoretical findings by demonstrating the effects of employing various initialization strategies on the model's robustness using benchmark adversarial attacks on real-world datasets.
Note that while our analysis primarily focuses on the widely used Graph Convolutional Networks (GCNs) \cite{Kipf:2017tc} and Graph Isomorphism Networks (GINs) \cite{xu2019powerful}, we highlight the versatility of our approach by providing a general upper bound applicable to any Deep Neural Networks in Section \ref{sec:generalization}. This underlines the potential for extending our analysis to a wide range of architectures, showcasing its broad applicability in understanding and enhancing adversarial robustness in neural networks. We summarize our contributions as follows:
\begin{itemize}
\item We provide a theoretical analysis that links weight initialization strategies with adversarial robustness in GNNs. We specifically derive an upper bound connecting a model's robustness to weight initialization and the number of training epochs, demonstrating that the initialization strategy can significantly influence the network's adversarial robustness. 
\item We validate our theoretical findings by conducting extensive experiments across various models using different benchmark adversarial attacks on real-world datasets. These experiments demonstrate that certain weight initialization strategies can enhance the model's defense against adversarial attacks, without degrading its performance on clean datasets. 

\item While our primary focus is on GNNs, we extend our analysis to Deep Neural Networks, illustrating the broader applicability of our theoretical analysis and its corresponding insights.
\end{itemize}

\section{Related Work}\label{sec:related_work}

\textbf{Graph Adversarial Attacks.}
Multiple studies focus on designing adversarial attacks capable of fooling a graph-based classifier \cite{gunnemann2022graph, pgd_paper, ennadir2023unboundattack}. The majority of these methods \cite{zugner_2019, zhan2021black} approach the adversarial aim as an optimization problem and employ different methods to solve it such as meta-learning \cite{zugner2019adversarial}. Furthermore, Nettack \cite{zugner2018adversarial} constrained the problem by preserving degree distribution and imposing constraints on feature co-occurrence to generate unnoticeable perturbations. Finally, reinforcement learning was proposed recently as a means to generate graph adversarial attacks \cite{dai2018}.

\textbf{Graph Adversarial Defenses.} 
Recent efforts have emerged to defend against the aforementioned adversarial attacks. In particular, methods such as low-rank matrix approximation coupled with graph anomaly detection \cite{Ma_2021} have been used. For example, GNN-Jaccard \cite{gnn_jaccard} proposed to pre-process the graph's adjacency matrix to detect potential manipulation of edges. Other methods such as edge pruning \cite{gnn_guard} and transfer learning \cite{Tang_2020} have been leveraged to limit the effect of poisoning attacks. Additionally, adaptations of the message-passing scheme, such as employing orthogonal weights~\cite{abbahaddou2024bounding} or introducing noise during training \cite{ennadir2024simple}, have been shown to perform well in terms of defense. Furthermore, there is a growing interest in exploring robustness certificates \cite{zugner_2019, bojchevski_2019} as a means of ensuring model robustness. For instance, \cite{bojchevski_certificate_2020} used randomized smoothing to provide a highly scalable model-agnostic certificate for graphs. Additionally, other robustness certificates for GCN-based graph classification under topological perturbations have been proposed \cite{certificate_jin_2020}. 

\textbf{Weight Initialization.} 
The impact of weight initialization has been extensively studied both theoretically and empirically where the main line of study consists of understanding the interplay between initialization techniques and the implicit regularization they induce, thereby elucidating their influence on a model's generalization capabilities \cite{xiao2018dynamical, pennington2018emergence}. For instance, it has been showcased that sampling initial weights from the orthogonal group can speed up convergence \cite{Hu2020Provable}. Similarly, alternative initialization approaches such as the Glorot Initialization \cite{glorot_init} and Kaiming Initialization \cite{he_init} have been proposed in efforts to improve the model's performance.

Our work stands apart from existing research on adversarial robustness as it represents, to the best of our knowledge, the first attempt to theoretically investigate the impact of initialization on a model's robustness. Moreover, our approach diverges fundamentally from existing literature on weight initialization as our focus lies in theoretically understanding the effect of initialization on a model's robustness rather than its implications for generalization or convergence.

\section{Graph Adversarial Robustness}\label{sec:graph_adv_attack}

In this section, we start by introducing the notation and some fundamental concepts related to GNNs. We afterwards establish the problem setup together with the set of considered assumptions. We finally lay out a GNN's robustness formalization on which we will build our theoretical analysis. 
\subsection{Preliminaries}\label{sec:preliminaries}

Let $G = \left(V,E\right)$ be a graph where $V$ ($\vert V \vert = n$) is its set of vertices and $E$ its set of edges. We denote $A \in \mathcal{A} \triangleq \lbrace 0,1 \rbrace^{n \times n}$ its adjacency matrix. The graph nodes are annotated with feature vectors $X \in \mathcal{X}\subseteq \mathbb{R}^{n\times d}$ (the $i$-th row of $X$ corresponds to the feature of node $i$). We denote by $\mathcal{N}(i)$ the neighbors of node $i \in V$ and $\Vert \cdot \Vert_2$ the Euclidean (resp., spectral) norm for vectors (resp., matrices).

In this work, we consider the task of node classification. In this task, every node is assigned exactly one class from $\mathcal{C} = \{ 1, 2, \ldots, C \} \subset \mathcal{Y}$ and we consider $d_\mathcal{Y}$ as a distance within the output space $\mathcal{Y}$. The learning objective is to find a function $f_{W}$, parameterized by $W$, that assigns each node $i \in V$ a class $c \in \mathcal{C}$ while minimizing some classification loss (\eg cross-entropy loss), denoted as $\mathcal{L}$.

\textbf{GNNs.}
A GNN model consists of a series of neighborhood aggregation layers that use the graph structure and the node features from the previous layers to generate new node representations.
Specifically, GNNs update node feature vectors by aggregating local neighborhood information. In the particular case of GCNs, this process is described by the following iterative propagation:
\begin{equation} \label{equation:gcn}
    h^{(\ell)} = \phi^{(\ell)}\left(\widehat{A}h^{(\ell-1)}W^{(\ell)}\right),
\end{equation}
with $W^{(\ell)} \in \mathbb{R}^{p \times q}$ being the weight matrix in the $\ell$-th layer, $p$ and $q$ are embedding dimensions and $\phi^{(\ell)}$ is a non-linear activation function. Moreover, $\widehat{A} \in \mathbb{R}^{n \times n}$ denotes the normalized adjacency matrix $\widehat{A} = D^{-1/2} A D^{-1/2},$ where $D=\text{diag}(\vert \mathcal{N}(1)\lvert, \vert \mathcal{N}(2)\lvert, \ldots, \vert \mathcal{N}(n)\vert)$ is the degree matrix.

\textbf{Problem Setup.} For our theoretical analysis, we assume that the model is based on 1-Lipschitz activation functions (which is a characteristic of commonly used activation functions such as tanh). Additionally, we consider the training loss function $\mathcal{L}$ to be $L$-smooth and that it is minimized using gradient descent. We denote by $W_{*}$ the local optimum towards which gradient descent iteratively converges. Specifically, for a learning rate $\eta \leq \frac{1}{L}$, the update at time step $t$ for a layer $i$ is:
\begin{equation*}
    W_{t+1}^{(i)} = W_{t}^{(i)} - \eta \nabla \mathcal{L}\left(W_{t}^{(i)}\right). 
\end{equation*}
It is worth emphasizing that although we focus on the node classification task, which is prevalent and well-studied in the literature of adversarial robustness, our analysis is equally applicable to other tasks such as graph classification. Moreover, while our theoretical analysis predominantly centers around using gradient descent as the optimizer, this choice does not limit the generality of our findings. One can employ a different optimizer and still yield the same insights and results by following a similar approach as the one outlined in this paper. Consequently, this specific setup should not be perceived as a limitation but rather as an analytical choice.

\subsection{Adversarial Robustness for Graph Neural Networks}
Let $f: (\mathcal{A}, \mathcal{X}) \rightarrow \mathcal{Y}$ be a GNN-classifier following the framework outlined in Section \ref{sec:preliminaries}. An adversarial attacks consists of generating an alternative graph $(\tilde{A},\tilde{X})$ that perturbs the original prediction $f(A,X)$ while not being far (semantically) from the original graph. Typically, this generated graph must adhere to a number of constraints related to its similarity to the original graph, defined by a perturbation budget $\epsilon$ controlling the number of edited edges or features. The set of these graphs is written as $B([A,X]; \epsilon) = \left\{(\tilde{A}, \tilde{X}):  \min_{\substack{P \in \Pi}} \left( \lVert A - P \tilde{A} P^T \rVert_{2} + \lVert X - P \tilde{X} \rVert_2\right) \leq \epsilon \right\}$, where $\Pi$ represents the set of permutations of the adjacency matrix. While the previous formulation relies on the $\ell_2$ norm, other norms may be used depending on the domain of application and the specific use case.
Building on previous work \cite{ennadir2024simple}, the adversarial risk of a GNN can be defined as the expected error of adjacent graphs within the considered graph's neighborhood defined by $\epsilon$ written as:
\begin{equation}\label{equation:robustness_definition}
\mathcal{R}_{\epsilon}[f] = \mathop{\mathbb{E}}_{\substack{(A, X) \sim \mathcal{D} }} \left[\sup_{(\tilde{A}, \tilde{X}) \in B([A,X]; \epsilon)} d_{\mathcal{Y}}\left(f\left(\tilde{A}, \tilde{X}\right), f\left(A, X\right)\right)\right].
\end{equation}

In the current analysis, we focus on the $\ell_2$ norm as our output distance $d_{\mathcal{Y}}$ (which can be substituted by any norm -- given the equivalence of norms). 
We theoretically approach the introduced adversarial risk by deriving an upper-bound, which reflects the model's expected error under input perturbation. Intuitively, a smaller upper bound reflects a smaller adversarial risk which in turn suggests a robust behavior locally. In this perspective, Definition \ref{def:robustness} draws the link between the considered risk quantity and a model's robustness. 
\begin{definition}
\label{def:robustness}
(Adversarial Robustness). The graph-based function $f: (\mathcal{A}, \mathcal{X}) \rightarrow \mathcal{Y}$ is said to be $(\epsilon, \gamma)-\text{robust}$ if its adversarial risk is upper-bounded by $\gamma$, \ie $\mathcal{R}_{\epsilon}[f] \leq \gamma$. 
\end{definition}

The current definition addresses adversarial risk from a worst-case scenario perspective, which is the most prevalent approach in the literature. This means we aim to identify the neighbor graph that maximizes the harm (\ie causes the greatest deviation from the original prediction). By upper-bounding the risk associated with this ``worst-case'' graph, we inherently account for all other potential adversaries within the same neighborhood, as their risk will be less than or equal to that of the worst-case scenario. We note that the nuances between the ‘‘average’’ and ‘‘worst-case’’ approaches have been thoroughly examined and justified in previous research \cite{average_and_worst_case_neurips21}.

\section{On the Effect of Initialization}\label{sec:main_section}

We start by considering the Graph Convolutional Networks (GCNs) within the broader context of Message Passing Neural Networks for node classification. This study investigates how initialization and other hyperparameters impact the final model's robustness. In this context, we aim to establish a connection between the introduced adversarial risk (Equation (\ref{equation:robustness_definition})) and the initial weight distribution and its evolution during training. Specifically, we seek to demonstrate that different choices in the initialization distribution and other relevant parameters lead to varying levels of model robustness, offering new insights into the potential trade-offs between initialization strategies and robustness. In this sense, we derive an upper-bound (denoted as $\gamma$ in Definition \ref{def:robustness}) on the stability of a GCN-based classifier when the input graph's node features are subject to adversarial attacks.

\begin{theorem}\label{theo:gcn_node_features}

Let $f: (\mathcal{A}, \mathcal{X}) \rightarrow \mathcal{Y}$ denote a graph-based function composed of $T$ GCN layers, where the initial weight matrix of the $i$-th layer is denoted by $W_0^{(i)}.$ For adversarial attacks only targeting node features of the input graph, with a budget $\epsilon$, we have (in respect to Definition \ref{def:robustness}):
\begin{align*}
    \gamma = \epsilon \prod_{i=1}^{T} \left( 2^t \left\lVert W_0^{(i)}  \right \lVert  + 2^{t+1} \left\lVert W_{*}^{(i)} \right \lVert \right)   \left(\sum_{u \in \mathcal{V}} \hat{w}_u \right)
\end{align*}

with $t$ being the number of training epochs and $\hat{w}_u$ denoting the sum of normalized walks of length $(T-1)$ starting from node $u.$
\end{theorem}

The proof of Theorem \ref{theo:gcn_node_features} is provided in Appendix \ref{appendix:proof_first_theo}. Theorem \ref{theo:gcn_node_features} provides a formal connection between the robustness of a GCN-based classifier and its initial weights, offering valuable insights into their effects. From a first perspective, the derived upper-bound depends on the initial weight's norm. Specifically, a lower norm corresponds to a smaller upper-bound, indicative of a more robust model. However, while setting all initial weights to zero theoretically yields the smallest upper-bound and consequently the optimum robustness, this direction can detrimentally affect the model's performance on the learning task. Empirical evidence suggests that initializing weights to zero (or a constant) often leads to poor learning outcomes, as it constrains weight behavior during propagation, limiting subsequent back-propagation operations and resulting in convergence to unsatisfactory local minima (\eg see Page 301 in \cite{goodfellow2016deep}).
From a second perspective, it appears that a higher number of training epochs leads to the looseness of the upper-bound, resulting in increased adversarial vulnerability. This latter observation provides proof and highlights the existence of the usually discussed trade-off between clean and attacked accuracy. Achieving a balance between increasing the number of epochs to achieve satisfactory clean accuracy and limiting them to attain a robust model is hence essential. While theoretically challenging to identify this equilibrium point, our experimental results demonstrate its existence. We note that the dependence of $\gamma$ on $t$ can be sharpened by having $(1+\eta L)^t$ instead of $2^t$. With small $\eta$ (which is usually the case in practice), $(1+\eta L)^t \approx 1+ t\eta L$ resulting in a bound which depends linearly on $t$. The same remark applies to the remaining bounds derived in the paper. These insights, in the case of node-feature-based adversarial attacks, also extend to structural perturbations where Theorem \ref{theo:structural_perturbations} provides the exact bound for this case.

\begin{theorem}\label{theo:structural_perturbations}
Let $f: (\mathcal{A}, \mathcal{X}) \rightarrow \mathcal{Y}$ denote a graph-based function composed of $T$ GCN layers, where the initial weight matrix of the $i$-th layer is denoted by $W_0^{(i)}.$ Let $f$ be the number of used training epochs. When $f$ is subject to structural attacks, with a budget $\epsilon$, we have (in respect to Definition \ref{def:robustness}):
\begin{align*}
    \gamma = \epsilon \prod_{i=1}^{T} \left( 2^t \left\lVert W_0^{(i)}  \right \lVert  + 2^{t+1} \left\lVert W_{*}^{(i)} \right \lVert \right) \lVert X \rVert  \left (1 + T \prod_{i=1}^{T} \left( 2^t \left\lVert W_0^{(i)}  \right \lVert  + 2^{t+1}  \left\lVert W_{*}^{(i)} \right \lVert \right) \right ).
\end{align*}
\end{theorem}

The computed upper-bound suggests that the effect of initialization is greater in the case of structural perturbations. This emphasis is resulting from the distinct dynamics within the message passing mechanism, where the influence of the adjacency matrix and node features varies during each propagation step. Precisely, for structural perturbations, the effect of the attack is considered at each propagation step through the perturbed adjacency matrix (in the aggregation step). Moreover, the impact is also amplified by the affected residual layers from previous iterations, resulting in a more significant attack result. This is different in the case of node-feature based adversarial attacks, since the node features are only directly taken into account in the first propagation.
Overall, the main takeaway of the provided analysis in Theorems \ref{theo:gcn_node_features} and \ref{theo:structural_perturbations} is that ‘‘approximately-free'' robustness enhancements can be derived from choosing the right initial weight's distribution and the right number of training epochs. 
We illustrate this specific point by analyzing the effect of the initial distributions choices on the model's robustness. Specifically, we consider the case of the Gaussian distribution, where Lemma~\ref{lemma:application_bound} studies how the parameters of this distribution – namely, the mean and variance – exert an influence on the expected (in respect to the initial distribution) value of the adversarial risk.

\begin{lemma} \label{lemma:application_bound}
Let $f: (\mathcal{A}, \mathcal{X}) \rightarrow \mathcal{Y}$ denote a graph-based function composed of $T$ GCN layers for which the initial weight are drawn from the Gaussian distribution $\mathcal{N}(\mu, \Sigma)$. When subject to node features based adversarial attacks, we have the following:
\begin{align*}
\mathop{\mathbb{E}_{\substack{W_0 \sim \mathcal{N}(\mu, \Sigma)}}}[\mathcal{R}_{\epsilon}[f]] \leq \epsilon \prod_{i=1}^{T} \left( 2^t \sqrt{ \mu^2 + \text{tr}(\Sigma)}   + 2^{t+1} \left\lVert W_{*}^{(i)} \right \lVert \right)   \left(\sum_{u \in \mathcal{V}} \hat{w}_u \right).
\end{align*}
\end{lemma}

The proof of Lemma \ref{lemma:application_bound} is provided in Appendix \ref{appendix:proof_lemma}. Given that a tighter upper bound inherently results in a higher level of robustness, the results derived in Lemma \ref{lemma:application_bound} illustrate the clear effect of initialization in the case of the Gaussian distribution. The derived bound shows that increasing the distribution parameters, both the mean and variance values, leads to a decrease in the victim model's underlying robustness. While one might intuitively aim to set these parameters as low as possible to achieve optimal robustness, doing so could potentially compromise the model's performance on clean datasets. Therefore, as previously mentioned, striking the right balance between clean accuracy and adversarial robustness is crucial.

\textbf{Extending the Results to the GIN.} The same previously applied analysis for the GCN-based models can be extended to take into account GIN-based classifiers. We consider the same set of assumptions and the same problem setup considered during the previously studied GCN case. We additionally assume that the input node feature space to be bounded, \ie $\lVert X \lVert \leq B$. We note that this boundedness is a realistic assumption and that the value $B$ can be easily computed for any real-world dataset. 

\begin{theorem}\label{theo:result_gin}
Let $f: (\mathcal{A}, \mathcal{X}) \rightarrow \mathcal{Y}$ denote a graph-based function composed of $T$ GIN layers, where the initial weight matrix of the $i$-th layer is denoted by $W_0^{(i)}.$ For adversarial attacks only targeting node features of the input graph, with a budget $\epsilon$, we have:
\begin{align*}
    \gamma = \prod_{l=1}^T \left(2^t \left \lVert W_{0}^{(i)} \right\lVert + 2^{t+1} \left\lVert W_{*}^{(i)}\right\lVert \right)  \left[B  T \max_{u \in \mathcal{V}}  deg(u) + \epsilon \right ]  
\end{align*}
with $t$ being the number of training epochs and $deg(u)$ is the degree of node $u.$
\end{theorem}

The proof of Theorem \ref{theo:result_gin} is provided in Appendix \ref{appendix:proof_gin}.  Theorem \ref{theo:result_gin} establishes an upper bound on the robustness of a GIN-based classifier against adversarial attacks targeting node features. We observe analogous insights, to the ones derived for a GCN-based classifier, regarding the influence of the initialization distribution and number of training epoch on the model's underlying robustness.

\section{Generalization to Other Models}\label{sec:generalization}

While our primary research focus lies within the domain of graph representation learning, a sub-field of the broader landscape of Deep Learning models, the fundamental principles of our theoretical analysis are applicable across various model architectures. Notably, and to our knowledge, the absence of a comparable study in current adversarial literature motivates our endeavor to bridge this gap. In this section, we aim to fill this gap by presenting a comprehensive analytical framework that provides the connection between weight initialization and the robustness of neural networks.

Let $x \in \mathbb{R}^{n_0}$ denote an input vector where $n_0$ is the input dimension. Let $W^{(l)} \in \mathbb{R}^{n_{l-1}, n_{l}}$ be the weight matrix and $b_l \in \mathbb{R}^{n_l}$ the bias of the $l^{\text{th}}$ layer with $n_l$ being its dimensionality. We focus on the general family of neural networks for which the computation during layer $l$, using an activation function $\phi^{(l)}$, can be written as :
\begin{align*}
    h^{(l)} = \phi^{(l)}\left( W^{(l)}h^{(l-1)} + b^{(l)}\right).
\end{align*}

We consider the same set of assumptions (stated in Section \ref{sec:preliminaries}) as the one from previous section. We consider the $\ell_2$ norm as our input and output distances within the metric space $\mathbb{R}^{n_0}$ and we consider an input attack budget $\epsilon$. The introduced adversarial risk in Equation \ref{equation:robustness_definition} can be easily extended and tailored to the family of considered neural networks discussed in this section. Further clarification on this extension is provided in the Appendix (Section \ref{appendix:robustness_of_DNNs}). From this standpoint, by adapting the Definition \ref{def:robustness}, analogous effects of the weight initialization, provided in Theorem \ref{theo:bound_neural_network}, can be observed.

\begin{theorem} \label{theo:bound_neural_network}
Let $f: \mathcal{X} \subseteq \mathbf{R}^{in} \rightarrow \mathcal{Y} \subseteq \mathbf{R}^{out}$ be a $T$-layers neural network with $W_0^{(i)}$ denoting the initial weight matrix of the $i$-th layer. When subject to adversarial attacks, $f$ is $(\epsilon, \gamma)-\text{robust}$ with:
\begin{align*}
    \gamma = \epsilon \prod_{i=1}^T \left(2^t \left\lVert W_0^{(i)}  \right\lVert  +  2^{t+1} \left\lVert W_{*}^{(i)}\right\lVert\right)
\end{align*}
\end{theorem}
The proof of Theorem \ref{theo:bound_neural_network} can be found in Section \ref{appendix:proof_neural_network} of the Appendix. Similar to previous findings, the upper bound relies on key elements of the initialization process, specifically the initial weight norm and the number of training epochs. These results validate and extend the established link between initialization and a model's robustness in neural networks, highlighting the importance of selecting appropriate parameters. 
From the derived upper bound, which is also applicable to GCN and GIN cases, we observe that the number of training epochs exerts an effect on the bound. Specifically, while increasing the number of epochs can improve the model’s performance on a clean dataset, it simultaneously leads to a deterioration in the model’s adversarial robustness. Ideally, adversarial defense strategies aim to avoid this trade-off between clean and attacked accuracy, striving for robust models that do not compromise the initial performance.
In this context, considering the strong-convexity of the loss function $\mathcal{L}$, in addition to the previously made assumptions, we observe that the effect of the number of training epochs becomes less pronounced. Lemma \ref{lemma:result_strong_convexity} specifically provides the computed bound under these assumptions.

\begin{lemma}\label{lemma:result_strong_convexity}
	Let $f: \mathcal{X} \subseteq \mathbf{R}^{in} \rightarrow \mathcal{Y} \subseteq \mathbf{R}^{out}$ be a $T$-layers neural network trained with a $\mu$-strongly convex and $L$-smooth loss function. Let $W_0^{(i)}$ denote the initial weight matrix of the $i$-th layer.  When subject to adversarial attacks, with a budget $\epsilon$, we have that $f$ is $(\epsilon, \gamma)-\text{robust}$ with:
	\begin{align*}
		\gamma =  \epsilon \prod_{i=1}^T \left( (1-\mu/L)^t \left\lVert W_{0}^{(i)}  \right\lVert + 2 \left\lVert W_{*}^{(i)}  \right\lVert \right ) 
	\end{align*}
\end{lemma}

The proof of Lemma \ref{lemma:result_strong_convexity} is provided in Section \ref{appendix:convex_case} of the Appendix.  
Since $\mu \leq L$, increasing the number of training epochs results in the diminishing influence of the initialization weights. 
In this scenario, the bound depends solely on the final weights, a phenomenon previously explored in works such as Parseval networks \cite{cisse2017parseval} for neural networks and GCORN \cite{abbahaddou2024bounding} for GNNs. This observation highlights the necessity of convexity in the loss function when training a neural network, as it plays a crucial role in enhancing the model's robustness, beyond the traditional considerations of classical training optimization perspectives.

\section{Experimental Results}\label{sec:experimental_results}
This section aims to empirically validate our theoretical findings using real-world benchmark datasets. We start by laying out our experimental setting, then we study the impact of various initialization strategies on a GCN's robustness. Next, we analyze the influence of training epochs on adversarial robustness. Finally, we extend our experimentation to considered family of DNNs in Section \ref{sec:generalization}. 

\subsection{Experimental Setting}
\textbf{Experimental Setup.} Consistent with our theoretical analysis, this section focuses on the node classification task. We leverage the citation networks Cora and CiteSeer \cite{dataset_node_classification}, with additional results on other datasets provided in the Appendix \ref{appendix:additional_results}. To mitigate the impact of randomness during training, each experiment was repeated 10 times, using the train/validation/test splits provided with the datasets. A 2-layers GCN classifier with identical hyperparameters and activation functions was employed across all the experiments. The models were trained using the cross-entropy loss function, and consistent values for the number of epochs and learning rate were maintained across all analysis. Further implementation details can be found in Appendix \ref{appendix:dataset_implementation_details}. The necessary code to reproduce all our experiments is available on github \href{https://github.com/Sennadir/Initialization_effect}{https://github.com/Sennadir/Initialization\_effect}.

\textbf{Adversarial Attacks.} We consider two main gradient-based structural adversarial attacks: \textbf{(i)} `Mettack' (with the ‘Meta-Self’ training strategy) \cite{zugner2019adversarial} that formulates the problem as a bi-level problem solved using meta-gradients \textbf{(ii)} and the Proximal Gradient Descent (PGD) \cite{pgd_paper} which consists of iteratively adding small crafted perturbations using the gradient of the classifier's loss. We additionally provide results for the `Dice' attack \cite{zugner2019adversarial} in Appendix \ref{appendix:additional_results}. For our experiments, we considered perturbation rates ranging from $10 \%$ (\ie $0.1 |E|$) to $40 \%$ (\ie $0.4 |E|$).

\textbf{Evaluation Metrics.} We report the experimental findings in terms of the ‘Attacked Accuracy’, which is the model's test accuracy when subject to the attacks. Additionally, given that initialization have an impact on the model's generalization and performance, solely reporting the attacked accuracy fails in some specific cases to provide a comprehensive perspective. Thus, we adopt for some experiments the ‘‘Success Rate’’ metric, also commonly employed in adversarial literature, which encompasses the number of successfully attacked nodes while taking into account the model's initial clean accuracy.

\subsection{Effect Of Training Epochs}
\begin{figure}[t]
    \centering
    \includegraphics[width=0.9\textwidth]{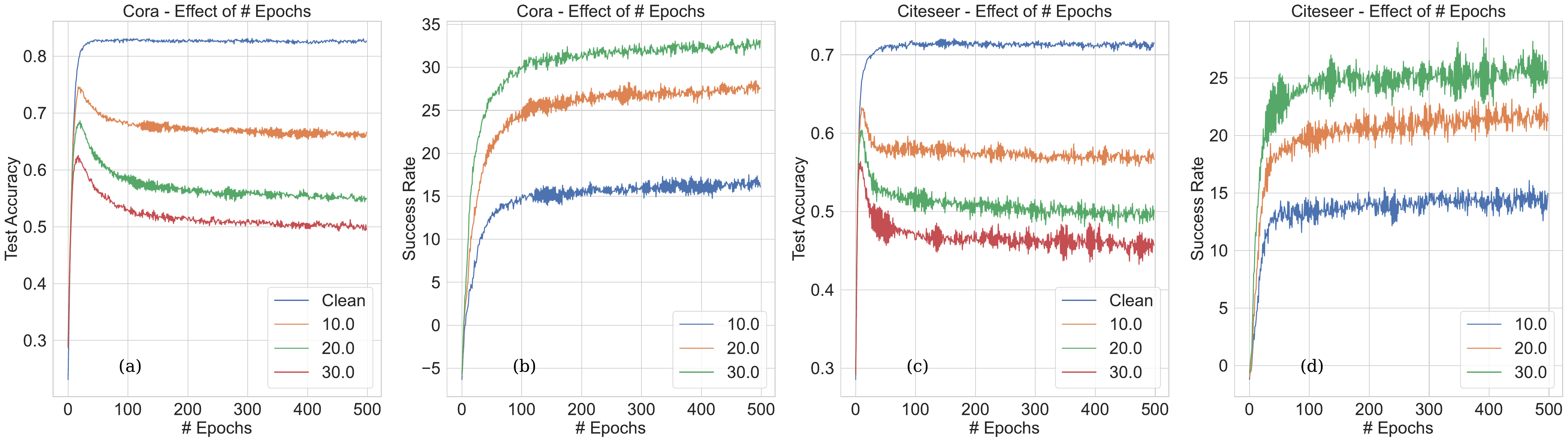}
    \caption{Effect of training epochs on the model's robustness on Cora (a,b) and CiteSeer (c,d).}
    \label{fig:epoch_effect}
\end{figure}

The theoretical analysis presented in Section \ref{sec:main_section} established a connection between the number of training epochs and the model's resulting robustness. The derived bound suggests that increasing the number of epochs results in the model becoming more vulnerable to adversarial attacks. The objective of this experimental section is to empirically validate this assertion using real-world datasets. To this end, at each training epoch, we assess the model's performance on the test set, considering both its clean accuracy and its accuracy under adversarial attacks.

Figure \ref{fig:epoch_effect} illustrates the results of this analysis. The initial two subplots (a,b) display the findings on the Cora dataset, while the subsequent (c,d) subplots present results from the CiteSeer dataset. For each dataset, the first plot showcases the clean and attacked accuracy, while the second plot shows the Success Rate (the discrepancy between the clean and attacked accuracy for each budget).
The experimental results demonstrate the existence of the previously discussed trade-off between clean and robust accuracies. Specifically, as anticipated, the clean accuracy exhibits a continual increase until reaching a plateau, corresponding to the convergence of the loss function to a minimum. Conversely, the attacked accuracy demonstrates a rising trend until reaching an inflection point, beyond which it begins to decline. These findings confirms the observations from the derived upper-bound, indicating that a higher number of epochs leads to increased vulnerability in the model. Ideally, users would aim to stop training at the inflection point, where the attacked accuracy is maximized while the clean accuracy remains proximal to its convergence point.

\begin{figure}[t]
    \centering
    \includegraphics[width=0.90\textwidth]{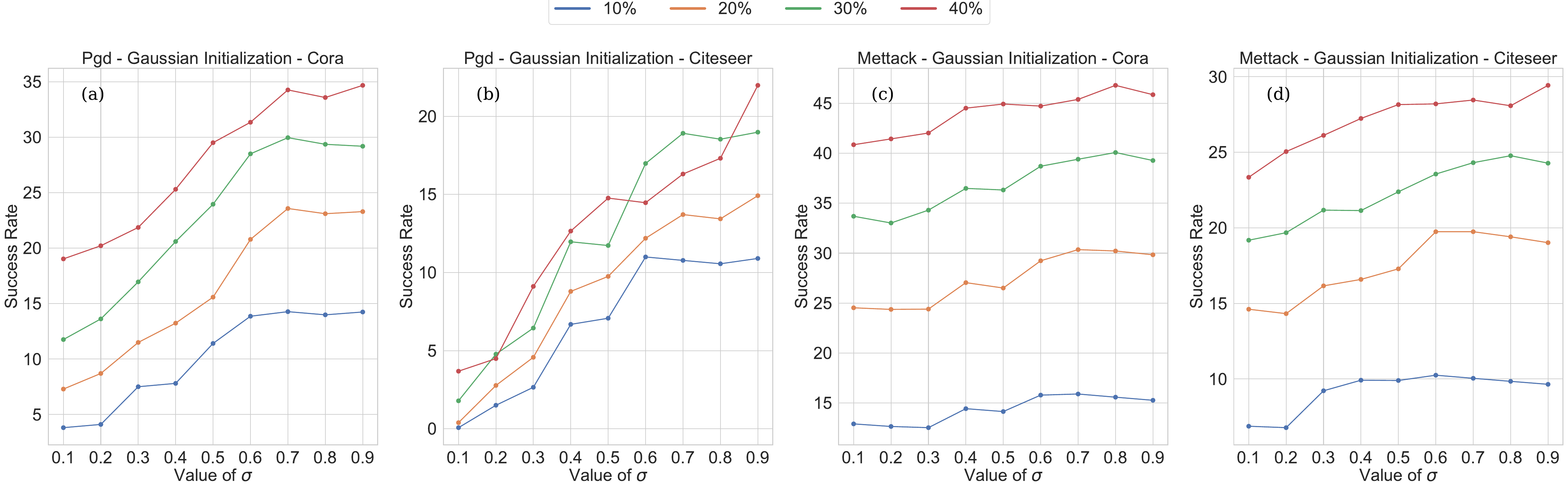}
    \caption{Effect of the variance parameter on the model's robustness in the case of Gaussian Initialization on PGD [on Cora (a) and Citeseer (b)] and Mettack [on Cora (a) and Citeseer (b)].}
    \label{fig:gaussian_init}
\end{figure}

\subsection{Effect Of Initial Weight Distribution}

We aim to validate the impact of the initial weight norms on the model's adversarial robustness. As previously discussed in Section \ref{sec:main_section}, a larger weight norm leads to the relaxation of the upper-bound, potentially resulting in the model being more susceptible to adversarial attacks.

In this perspective, we start by investigating the effect of sampling from a Gaussian distribution, as studied in Lemma \ref{lemma:application_bound}. We hence consider this latter by setting the mean value $\mu$ to a constant, and analyzing the impact of the variance parameter $\sigma$. Intuitively, based on the upper-bound analysis, a higher variance value is anticipated to result in reduced model robustness. Figure \ref{fig:gaussian_init} illustrates the resulting Success Rate across various variance values for both the ``PGD'' and ``Mettack'' methods, applied to the Cora and Citeseer datasets. The findings unequivocally validate the theoretical insights, demonstrating a direct correlation between increasing the variance ($\sigma$) and a higher Success Rates, indicating heightened vulnerability and reduced robustness of the model. Moreover, the impact of initialization becomes more pronounced when considering larger attack budgets, as outlined in the computed upper-bound. Notably, for certain budgets (e.g., $30\%$ and $40\%$), the observed gap ranges between $5\%$ and $15\%$, underscoring the initial weights significant implications on the robustness.

\begin{figure}[t]
    \centering
    \includegraphics[width=0.85\textwidth]{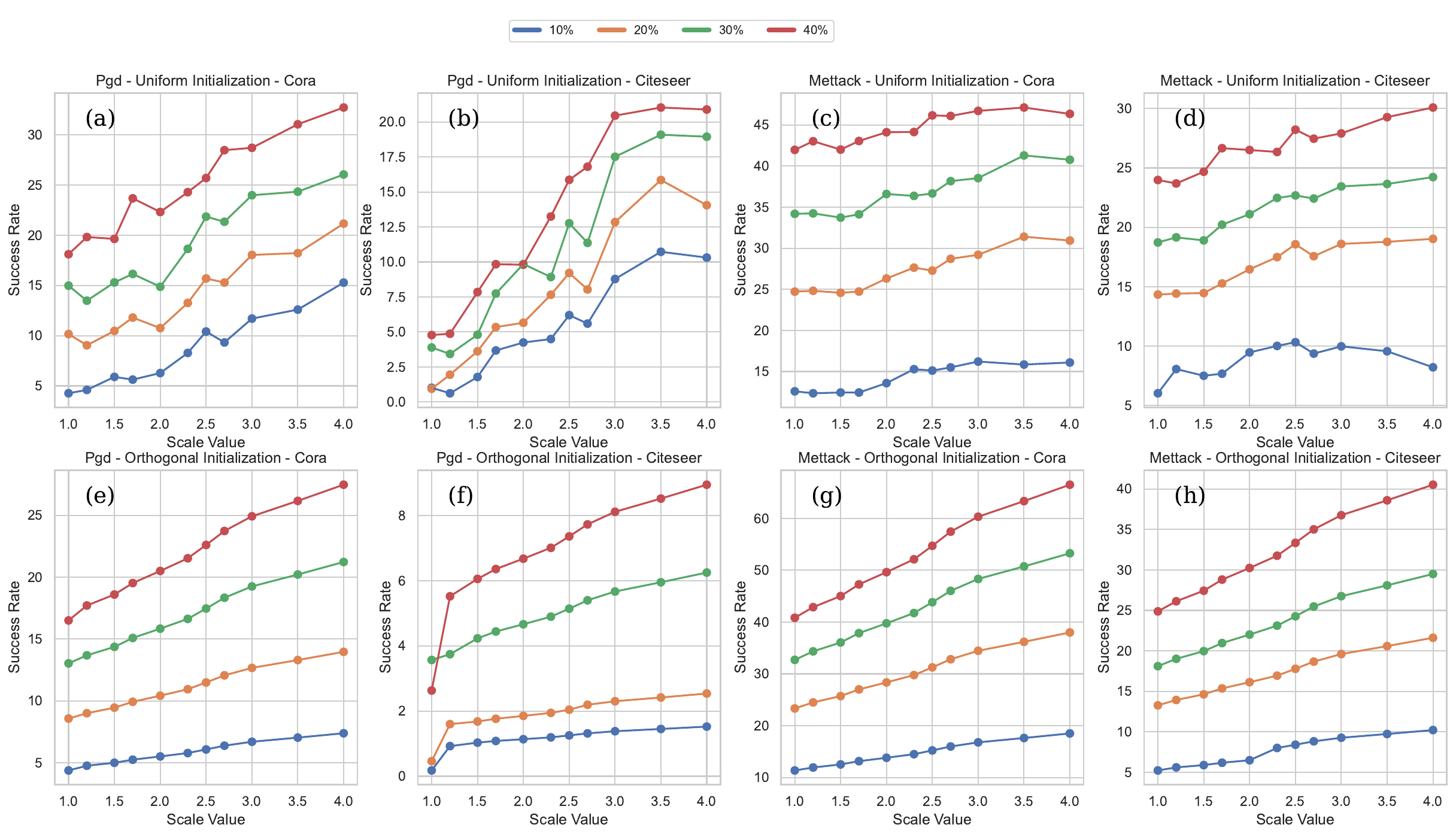}
    \caption{Effect of the scaling parameter $\beta$ on the model's robustness in the case of Uniform (a-d) and Orthogonal (e-h) Initialization when subject to PGD and Mettack using Cora and CiteSeer.}
    \label{fig:uniform_ortho}
\end{figure}

Within the same context, we explore alternative initialization strategies, focusing on two primary cases. First, we investigate sampling initial weights from a uniform distribution $\mathcal{U}(-\beta, \beta)$, where $\beta$ can be seen as a scaling parameter for weight norms. Second, we consider employing a scaled orthogonal weight initialization strategy. While this our aim can be approached by sampling weights from a scaled random Gaussian distribution, we adopt the orthogonal initialization strategy proposed in prior work \cite{saxe2014exact}, which we further rescale by a factor $\beta$ to examine the impact on weight norms. In both cases, higher scaling parameter values of $\beta$ are anticipated to theoretically yield higher upper-bounds and consequently render the model more vulnerable, as indicated by our computed bounds. We conduct numerical computations on both the Cora and Citeseer datasets to assess the resulting adversarial robustness of a GCN across various $\beta$ values, as provided in Figure \ref{fig:uniform_ortho}. The experimental results are exactly aligned with our theoretical findings showcasing the effect of the weight norm in the adversarial robustness. To summarize, while traditionally overlooked in prior studies on adversarial robustness, our experimentation underscores the critical importance of selecting appropriate initialization distributions and strategies for enhancing model robustness.

\subsection{Experimental Generalization}

\begin{wrapfigure}{r}{9.5cm}
\caption{Effect of initialization on the GIN (a) and DNN (b) for different attack budgets.}
\includegraphics[width=9.5cm]{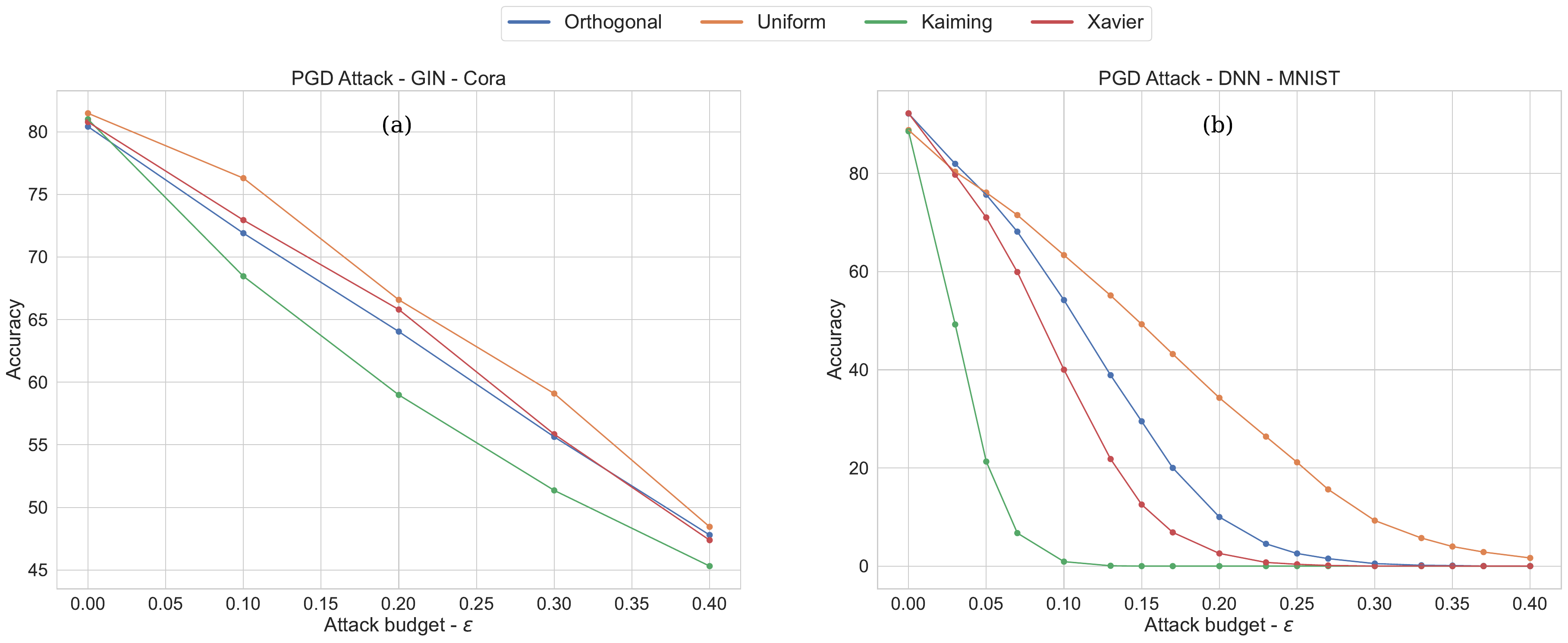}\label{fig:other_inits}
\end{wrapfigure} 

We extend our experimentation to empirically validate the theoretical generalizations provided in both Section \ref{sec:main_section} for the GINs and Section \ref{sec:generalization} for a DNNs. To this end, we consider these two models with various initialization schemes, including the previously used Orthogonal \cite{saxe2014exact} and Uniform initialization in addition to the Kaiming \cite{he_init} and Xavier Initialization \cite{glorot_init}. Our analysis primarily focuses on the PGD adversarial attack, using identical attack budgets as in the previous sections.
Figure \ref{fig:other_inits} presents the results on the GIN (a) using the Cora dataset and (b) on the DNN using the MNIST dataset. Notably, we observe that the different initialization methods yield similar clean accuracy ($\epsilon=0$), yet as the attack budget increases, the discrepancy in attacked accuracy between them also grows. For instance, in the case of DNNs, the accuracy gap between the best and worst initialization methods for $\epsilon=0.1$ ranges around $60\%$, proving our main assumption related to the impact of initialization on the model's robustness.

\section{Conclusion \& Limitations} 
The current study shows that the dynamics of learning in GNNs and DNNs have an important effect on the model's final robustness. Specifically, we theoretically showed that the model's robustness is connected to the weight initialization and the number of training epochs. We empirically validate our findings, where we can see that choosing the right initialization can yield huge ``almost-free'' robustness improvement. We additionally showed the existence of a trade-off between choosing the right number of epochs to have the best clean accuracy and the most robust model. While the current work did not propose an alternative or a solution, it has introduced a new perspective, which to our knowledge, was absent from the adversarial literature, opening the door to new research direction either by proposing new initialization schemes to improve robustness while guaranteeing good generalization or new gradient-based weight updates to enforce the robustness of the model or yet again by tracking robustness metrics alongside the loss function throughout training. 

\section*{Acknowledgements}
This work was partially supported by the Wallenberg AI, Autonomous Systems and Software Program (WASP) funded by the Knut and Alice Wallenberg Foundation. The computation (on GPUs) was enabled by resources provided by the National Academic Infrastructure for Supercomputing in Sweden (NAISS) at Alvis partially funded by the Swedish Research Council through grant agreement no. ``2024/22-309''. We furthermore want to thank Dr. Yassir Jedra for revising the manuscript and for a very helpful discussion on its different elements.

\bibliographystyle{plain}
\bibliography{references}


\newpage

\setcounter{page}{1}

\appendix

\vbox{%
\hsize\textwidth
\linewidth\hsize
\vskip 0.1in
\centering
{\LARGE\bf Supplementary Material: If You Want to Be Robust, Be Wary of Initialization\par}
\vspace{2\baselineskip}
}

\section{Proof of Theorem \ref{theo:gcn_node_features}}\label{appendix:proof_first_theo}

\begin{theorem*}

Let $f: (\mathcal{A}, \mathcal{X}) \rightarrow \mathcal{Y}$ denote a graph-based function composed of $T$ GCN layers, where the initial weight matrix of the $i$-th layer is denoted by $W_0^{(i)}.$ For adversarial attacks only targeting node features of the input graph, with a budget $\epsilon$, we have (in respect to Definition \ref{def:robustness}):
\begin{align*}
    \gamma = \epsilon \prod_{i=1}^{T} \left( 2^t \left\lVert W_0^{(i)}  \right \lVert  + 2^{t+1} \left\lVert W_{*}^{(i)} \right \lVert \right)   \left(\sum_{u \in \mathcal{V}} \hat{w}_u \right).
\end{align*}

with $t$ being the number of training epochs and $\hat{w}_u$ denoting the sum of normalized walks of length $(T-1)$ starting from node $u.$

\end{theorem*}

\begin{proof}
Let's consider a graph-function $f$ that is based on $T$ GCN-layers. The gradient descent update at epoch $t$ for a layer $i$ is written as:

\begin{align*}
    W_{t+1}^{(i)} = W_{t}^{(i)} - \eta \nabla \mathcal{L}(W_{t}^{(i)}).    
\end{align*}

Since we consider that our loss function $\mathcal{L}$ to be $L$-smooth, we have the following result:
\begin{align*}
    \left \lVert \nabla \mathcal{L}(W_{t}^{(i)}) \right \lVert \leq L \left\lVert W_{t}^{(i)} - W_{*}^{(i)}  \right\lVert.
\end{align*}

Consequently, after $t$ training epochs, we can write:
\begin{align*}
    \left \lVert W_t^{(i)}  \right\lVert &= \left \lVert W_{t-1}^{(i)} - \eta \nabla \mathcal{L}(W_{t-1}^{(i)})  \right \lVert \\
    & \leq \left \lVert W_{t-1}^{(i)} \right \lVert + \eta L \left \lVert W_{t-1}^{(i)} - W_{*}^{(i)} \right\lVert \\
    & \leq \left (1 + \eta L\right ) \left \lVert W_{t-1}^{(i)} \right\lVert + \eta L \left 
 \lVert W_{*}^{(i)}\right\lVert.
\end{align*}

In addition, we have that $\eta \leq \frac{1}{L}$. Hence, by recursion, we find that:

\begin{align}\label{eq:result_gd}
    \left \lVert W_t^{(i)}  \right \lVert     & \leq \left (1 + \eta L\right)^t \left \lVert W_{0}^{(i)} \right \lVert + \sum_{h=0}^{t} 2^h \left  \lVert W_{*}^{(i)}\right\lVert \\
    &\leq \left (1 + \eta L\right)^t \left \lVert W_{0}^{(i)} \right\lVert + 2^{t+1}\left  \lVert W_{*}^{(i)}\right\lVert. 
\end{align}

Giving that we are considering feature-based adversarial attacks, let $X$ denote the original node features and $X'$ denote the perturbed adversarial features. With an attack budget $\epsilon$, from the work \cite{abbahaddou2024bounding}, we have the following result:
\begin{align}
    \forall [A, X'] \in B\left([A,X], \epsilon\right), \left\lVert f(A, X) - f(A, X')\right\lVert \leq \prod_{i=1}^{T} \left\lVert W_t^{(i)} \right \rVert  \epsilon \left(\sum_{u \in \mathcal{V}} \hat{w}_u \right).
\end{align}

with $\hat{w}_u$ denoting the sum of normalized walks of length $(T-1)$ starting from node $u.$ Consequently:
\begin{align}\label{eq:result_adv_attacks}
	\sup_{[A, X'] \in B([A,X], \epsilon)}  \left \lVert f(A, X) - f(A, X')\right\lVert \leq \prod_{i=1}^{T} \left \lVert W_t^{(i)} \right\rVert  \epsilon \left(\sum_{u \in \mathcal{V}} \hat{w}_u \right).
\end{align}

From Equations (\ref{eq:result_gd}) and (\ref{eq:result_adv_attacks}), we conclude that:
\begin{align*}
    \sup_{[A, X'] \in B([A,X], \epsilon)}  \lVert f(A, X) - f(A, X')\lVert \leq  \epsilon \prod_{i=1}^{T} \left[2^t \left\lVert W_{0}^{(i)} \right\lVert + 2^{t+1} \left\lVert W_{*}^{(i)}\right\lVert \right]   \left(\sum_{u \in \mathcal{V}} \hat{w}_u \right).
\end{align*}
We conclude that $f$ is ($\epsilon$;$\gamma$)-robust with:
\begin{align*}
    \gamma = \epsilon \prod_{i=1}^{T} \left( 2^t \left\lVert W_0^{(i)}  \right \lVert  + 2^{t+1} \left\lVert W_{*}^{(i)} \right \lVert \right)   \left(\sum_{u \in \mathcal{V}} \hat{w}_u \right).
\end{align*}

\end{proof}

\section{Proof of Theorem \ref{theo:structural_perturbations}} \label{appendix:proof_Second_theo}

\begin{theorem*}

Let $f: (\mathcal{A}, \mathcal{X}) \rightarrow \mathcal{Y}$ denote a graph-based function composed of $T$ GCN layers, where the initial weight matrix of the $i$-th layer is denoted by $W_0^{(i)}.$ Let $f$ be the number of used training epochs. When $f$ is subject to structural attacks, with a budget $\epsilon$, we have (in respect to Definition \ref{def:robustness}):
\begin{align*}
    \gamma = \epsilon \prod_{i=1}^{T} \left( 2^t \left\lVert W_0^{(i)}  \right \lVert  + 2^{t+1} \left\lVert W_{*}^{(i)} \right \lVert \right) \lVert X \rVert  \left (1 + T \prod_{i=1}^{T} \left( 2^t \left\lVert W_0^{(i)}  \right \lVert  + 2^{t+1}  \left\lVert W_{*}^{(i)} \right \lVert \right) \right ).
\end{align*}
    
\end{theorem*}

\begin{proof}
Similar to the previous proof, let's consider a graph-function $f$ that is based on $T$ GCN-layers and trained using gradient descent for $t$ epochs. We have the following result from Equation \ref{eq:result_gd}:
\begin{align}
    \left\lVert W_t^{(i)}  \right\lVert \leq 2^t \left\lVert W_{0}^{(i)} \right\lVert + 2^{t+1} \left\lVert W_{*}^{(i)}\right\lVert.
\end{align}

For this proof, we are considering the model $f$ to be subject to structural perturbations. In this perspective, let $\Tilde{A}$ denote the input non-attacked adjacency and $\Tilde{A'}$ denote the attacked/perturbed adjacency, with $h'$ denoting its corresponding hidden representation. From the work \cite{abbahaddou2024bounding}, we have:
\begin{align*}
    \forall [A', X] \in B([A,X], \epsilon), \lVert f(\Tilde{A}, X) - f(\Tilde{A'}, X)\lVert \leq \prod_{i=1}^{T} \left\lVert W^{(i)} \right\rVert \left\lVert X \right\rVert \epsilon \left(1 + T  \prod_{i=1}^{T} \left\lVert W^{(i)} \right\lVert \right).
\end{align*}

By combining the two previous results, we get the following inequality and hence the desired result:
\begin{align*}
    \sup_{[A', X] \in B([A,X], \epsilon)} \lVert f(\Tilde{A}, X) - f(\Tilde{A'}, X)\lVert \leq & \epsilon \prod_{i=1}^{T} \left( 2^t \left\lVert W_0^{(i)}  \right \lVert  + 2^{t+1} \left\lVert W_{*}^{(i)} \right \lVert \right) \lVert X \rVert  \\ & \left (1 + T \prod_{i=1}^{T} \left( 2^t \left\lVert W_0^{(i)}  \right \lVert  + 2^{t+1}  \left\lVert W_{*}^{(i)} \right \lVert \right) \right )  .
\end{align*}

\end{proof}

\section{Proof of Lemma \ref{lemma:application_bound}}\label{appendix:proof_lemma}
\begin{lemma*}
Let $f: (\mathcal{A}, \mathcal{X}) \rightarrow \mathcal{Y}$ denote a graph-based function composed of $T$ GCN layers for which the initial weight are drawn from the Gaussian distribution $\mathcal{N}(\mu, \Sigma)$. When subject to node features based adversarial attacks, we have the following:
\begin{align*}
\mathop{\mathbb{E}_{\substack{W_0 \sim \mathcal{N}(\mu, \Sigma)}}}[\mathcal{R}_{\epsilon}[f]] \leq \epsilon \prod_{i=1}^{T} \left( 2^t \sqrt{ \mu^2 + \text{tr}(\Sigma)}   + 2^{t+1} \left\lVert W_{*}^{(i)} \right \lVert \right)   \left(\sum_{u \in \mathcal{V}} \hat{w}_u \right).
\end{align*}
\end{lemma*}

\begin{proof}
Let us consider $f$ to be a graph classifier based on  $T$-GCN layers for which the initial weight are drawn from the Gaussian distribution. Specifically, $\forall i \leq L, W_0^{(i)} \sim \mathcal{N}(\mu, \Sigma)$. We have that:
\begin{align*}
    \mathbb{E}\left[\lVert W_0^{(i)} \right\lVert] \leq \sqrt{\lVert \mu \lVert^2 + \text{tr}(\Sigma)}.
\end{align*}
From Theorem \ref{theo:gcn_node_features}, we have the following:
\begin{align*}
    \gamma = \epsilon \prod_{i=1}^{T} \left( 2^t \left\lVert W_0^{(i)}  \right \lVert  + 2^{t+1} \left\lVert W_{*}^{(i)} \right \lVert \right)   \left(\sum_{u \in \mathcal{V}} \hat{w}_u \right).
\end{align*}

Hence, combining the two elements results in the following:

\begin{align*}
\mathop{\mathbb{E}_{\substack{W_0 \sim \mathcal{N}(\mu, \Sigma)}}}[\mathcal{R}_{\epsilon}[f]] \leq \epsilon \prod_{i=1}^{T} \left( 2^t \sqrt{ \mu^2 + \text{tr}(\Sigma)}   + 2^{t+1} \left\lVert W_{*}^{(i)} \right \lVert \right)   \left(\sum_{u \in \mathcal{V}} \hat{w}_u \right).
\end{align*}

\end{proof}

\section{Proof of Theorem \ref{theo:result_gin}} \label{appendix:proof_gin}

\begin{theorem*}

Let $f: (\mathcal{A}, \mathcal{X}) \rightarrow \mathcal{Y}$ denote a graph-based function composed of $T$ GIN layers, where the initial weight matrix of the $i$-th layer is denoted by $W_0^{(i)}.$ For adversarial attacks only targeting node features of the input graph, with a budget $\epsilon$, we have:
\begin{align*}
    \gamma = \prod_{l=1}^T \left(2^t \left\lVert W_{0}^{(i)} \right\lVert + 2^{t+1} \left \lVert W_{*}^{(i)}\right\lVert \right)  \left[B  T \max_{u \in \mathcal{V}}  deg(u) + \epsilon \right ].  
\end{align*}
with $t$ being the number of training epochs and $deg(u)$ is the degree of node $u.$
    
\end{theorem*}

\begin{proof}

Let's consider a graph-function $f$ that is based on $T$ GIN-layers and trained using gradient descent for $t$ epochs. We have the following result from Equation \ref{eq:result_gd}:

\begin{align}\label{eq:result_gd_structure}
  \left \lVert W_t^{(i)}  \right \lVert    
    &\leq \left(1 + \eta L\right)^t \left \lVert W_{0}^{(i)} \right \lVert + 2^{t+1} \left \lVert W_{*}^{(i)}\right \lVert \leq       2^t \left\lVert W_{0}^{(i)} \right\lVert + 2^{t+1} \left\lVert W_{*}^{(i)}\right\lVert. 
    \end{align}

Let $X$ denote the original node features and $X'$ the perturbed adversarial features. For an attack budget $\epsilon$, from the work \cite{abbahaddou2024bounding}, we have the following: 

\begin{align}
    \forall [A', X] \in B([A,X], \epsilon), \lVert f(A, X) - f(A, X') \lVert \leq \prod_{l=1}^T \left\lVert W^{(l)}\right\lVert  \left [B  T  \max_{u \in \mathcal{V}}  deg(u) + \epsilon \right]. 
\end{align}

Consequently, we can merge the two inequalities resulting in the following:

\begin{align*}
    \gamma = \prod_{l=1}^T \left(2^t \left\lVert W_{0}^{(i)} \right \lVert + 2^{t+1} \left \lVert W_{*}^{(i)}\right \lVert \right)  \left[B  T  \max_{u \in \mathcal{V}}  deg(u) + \epsilon \right ]. 
\end{align*}
    
\end{proof}

\section{Proof of Theorem \ref{theo:bound_neural_network}} \label{appendix:proof_neural_network}

\begin{theorem*}
Let $f: \mathcal{X} \subseteq \mathbf{R}^{in} \rightarrow \mathcal{Y} \subseteq \mathbf{R}^{out}$ be a $T$-layers neural network with $W_0^{(i)}$ denoting the initial weight matrix of the $i$-th layer. When subject to adversarial attacks, $f$ is $(\epsilon, \gamma)-\text{robust}$ with:
\begin{align*}
    \gamma = \epsilon \prod_{i=1}^T \left(2^t \left \lVert W_0^{(i)}  \right \lVert  +  2^{t+1} \left \lVert W_{*}^{(i)}\right \lVert\right).
\end{align*}
\end{theorem*}

\begin{proof}
Let $f$ be a $T$-layers neural network. We additionally assume that its corresponding activation functions are 1-Lipschitz. Let $x$ (with $h$ its hidden representation) be an input vector and $x'$ (corresp. $h'$) its corresponding crafted adversarial input (corresp. hidden representation). For an adversarial attack with budget $\epsilon$, we have the following:

\begin{align*}
    \forall x' \in \mathcal{X} : \Vert x - x' \Vert \leq \epsilon, \lVert f(x) - f(x') \lVert &= \left \lVert h^{(l)} - h'^{(l)} \right \lVert\\
    &=\left \lVert \phi^{(l)}\left (W^{(l)}h^{(l-1)} + b^{(l)}\right) - \phi^{(l)}\left (W^{(l)}h'^{(l-1)} + b^{(l)}\right)  \right \lVert \\
    & \leq \left \lVert W^{(l)} \right \lVert \left \lVert h^{(l-1)} - h'^{(l-1)}  \right \lVert. 
\end{align*}

Recurrently, we find the final result as:
\begin{align}\label{eq:result_general_dnn}
    \sup_{x' \in \mathcal{X} : \Vert x - x' \Vert \leq \epsilon }   \lVert f(x) - f(x') \lVert \leq \prod_{l=1}^T \left \lVert W^{(l)}\right \lVert \epsilon.
\end{align}

Note that similar results and analysis have been provided in previous work \cite{cisse2017parseval, anil2019sorting}. By using the result derived in Equation \ref{eq:result_gd}, we have:
\begin{align}
	\left \lVert W_t^{(i)}  \right \lVert \leq 2^t \left \lVert W_{0}^{(i)} \right \lVert + 2^{t+1} \left \lVert W_{*}^{(i)}\right \lVert.
\end{align}

By merging these two inequalities, and applying the Markov Inequality, we find the following upper-bound:
\begin{align*}
    \gamma = \epsilon \prod_{i=1}^T \left(2^t \left \lVert W_0^{(i)}  \right \lVert  +  2^{t+1} \left \lVert W_{*}^{(i)}\right \lVert\right).
\end{align*}

\end{proof}

\section{On the Case of Strong-Convexity - Proof of Lemma \ref{lemma:result_strong_convexity}} \label{appendix:convex_case}

\begin{lemma*}
	Let $f: \mathcal{X} \subseteq \mathbf{R}^{in} \rightarrow \mathcal{Y} \subseteq \mathbf{R}^{out}$ be a $T$-layers neural network trained with a $\mu$-strongly convex and $L$-smooth loss function. Let $W_0^{(i)}$ denote the initial weight matrix of the $i$-th layer.  When subject to adversarial attacks, with a budget $\epsilon$, we have that $f$ is $(\epsilon, \gamma)-\text{robust}$ with:
	\begin{align*}
		\gamma =  \epsilon \prod_{i=1}^T \left( (1-\mu/L)^t \left \lVert W_{0}^{(i)}  \right\lVert + 2 \left \lVert W_{*}^{(i)}  \right\lVert \right ). 
	\end{align*}
\end{lemma*}

\begin{proof}

We consider $f$ to be a $T$-layers neural network (following the same propagation as equation the one presented in Section \ref{sec:generalization}). From Section \ref{appendix:proof_neural_network}, we have the following:
\begin{align*}
	\lVert f(x) - f(x') \lVert \leq \prod_{l=1}^T \left \lVert W^{(l)}\right \lVert \epsilon.
\end{align*}

In addition to the previous assumption of $L$-smoothness of the loss function, we consider that its $\mu$-strongly convex. Hence, for the layer $(l)$, we have the following result:
\begin{align}\label{eq:result_strong_conv}
	\left \lVert W_t^{(l)} \right \lVert &\leq \left (1-\mu/L\right )^t \left \lVert W_0^{(l)} - W_*^{(l)} \right\lVert + \left \lVert W_*^{(l)} \right \lVert \\
	&\leq \left (1-\mu/L\right )^t \left \lVert W_0^{(l)} \right \lVert + 2 \left \lVert W_*^{(l)} \right \lVert.  \label{eq:result_strong_convexity}
\end{align}

When subject to adversarial attacks, we can use the previous result from \ref{appendix:proof_neural_network}, specifically from Equation~(\ref{eq:result_general_dnn}):
\begin{align}
	\sup_{x' \in \mathcal{X} : \Vert x - x' \Vert \leq \epsilon }   \lVert f(x) - f(x') \lVert \leq \prod_{l=1}^T \left \lVert W^{(l)}\right \lVert \epsilon.
\end{align}
 
Hence, by merging the two previous results, we deduce that: 

\begin{align}
		\gamma =  \epsilon \prod_{i=1}^T \left( (1-\mu/L)^t \left \lVert W_{0}^{(i)}  \right\lVert + 2 \left \lVert W_{*}^{(i)}  \right \lVert \right ). 
\end{align}

\end{proof}

\begin{figure}[h]
    \centering
    \includegraphics[width=1.0\textwidth]{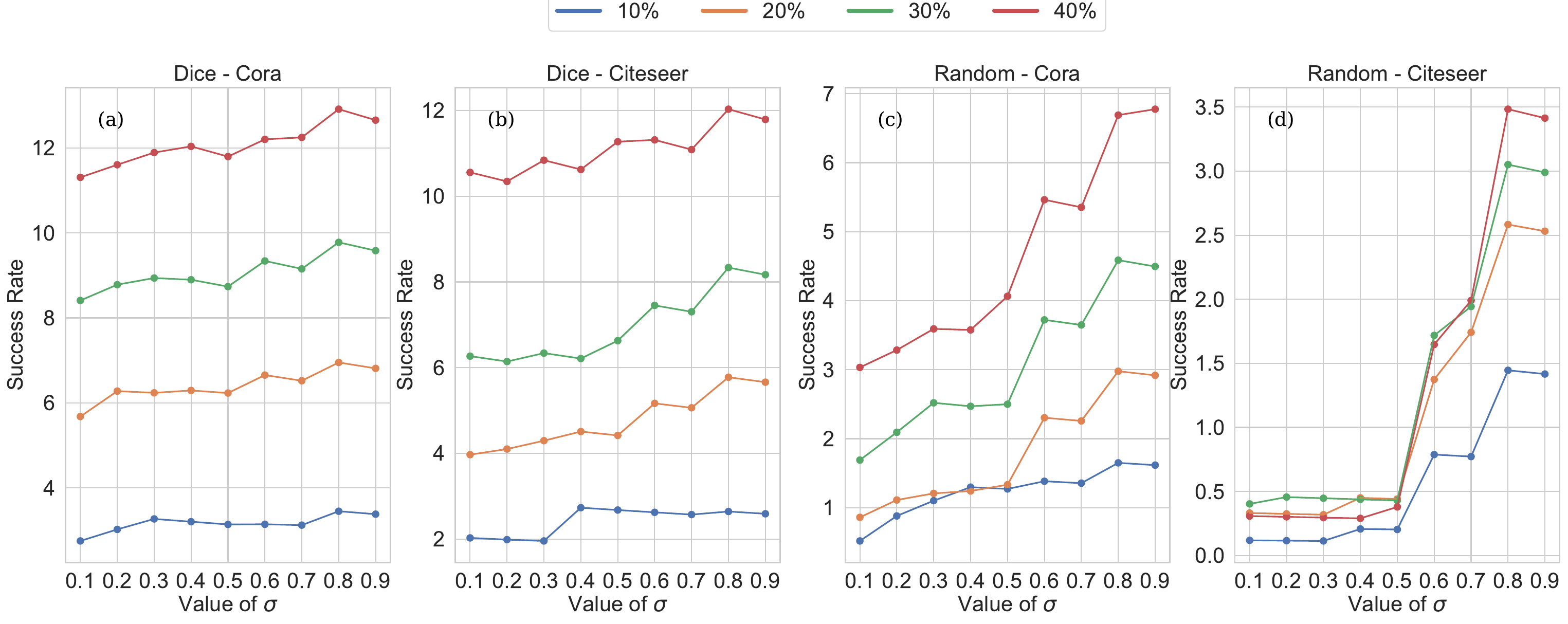}
    \caption{Effect of the variance on the model's robustness in the case of Gaussian Initialization when subject to DICE (a,b) and Random Attacks (c,d) for both Cora and CiteSeer.}
    \label{fig:dice_gaussian}
\end{figure}

\section{Additional Results}\label{appendix:additional_results}

\subsection{Adversarial Robustness of Deep Neural Networks}\label{appendix:robustness_of_DNNs}

We consider the general family of neural networks for which the computation during layer $l$, using an activation function $\phi^{(l)}$, can be written as :
\begin{align*}
    h^{(l)} = \phi^{(l)}(W^{(l)}h^{(l-1)} + b^{(l)}).
\end{align*}

with $W^{(l)} \in \mathbb{R}^{n_{l-1}, n_{l}}$ being the weight matrix and $b_l \in \mathbb{R}^{n_l}$ the bias of the $l^{\text{th}}$ layer. 

In this perspective, let $f: \mathbb{R}^{n_0} \rightarrow \mathbb{R}$ be a neural network $n_0$ being the input dimension. The adversarial task in this case consists of finding a perturbed input $\tilde{x}$ for which the prediction differs from the original prediction $f(x)$. The perturbed input $\tilde{x}$ should hence adhere to the similarity constraints defined by a perturbation budget $\epsilon$. Let's consider the $\ell_2$ norm within both the input space $\mathbb{R}^{n_0}$ and the output space $\mathbb{R}$, we can hence define the set of valid adversarial perturbation as: 

\begin{center}
    $B(x; \epsilon) = \{\tilde{x}: \lVert x - \tilde{x} \lVert \leq \epsilon \}.$
\end{center}

Similar to Section \ref{sec:graph_adv_attack}, we can introduce the adversarial risk of a DNN within the input's neighborhood defined by the budget $\epsilon$ as the following:
\begin{equation}\label{equation:DNN_risk}
\mathcal{R}_{\epsilon}[f] = \mathop{\mathbb{E}}_{\substack{x \sim \mathcal{D} }} \left[\sup_{\tilde{x} \in B(x; \epsilon)} \lVert (f(\tilde{x}) -  f(x) \lVert\right].
\end{equation}

From this adapted adversarial risk, we can introduce the notion of a DNN's adversarial robustness

\begin{definition}
(DNN - Adversarial Robustness). The neural network $f: \mathbb{R}^{n_0} \rightarrow \mathbb{R}$ is said to be $(\epsilon, \gamma)-\text{robust}$ if its adversarial risk is upper-bounded by $\gamma$, \ie $\mathcal{R}_{\epsilon}[f] \leq \gamma$.
\end{definition}

\subsection{Additional Adversarial Attacks}

In addition to the previously reported Mettack and PGD adversarial attack, we consider two additional adversarial attacks. Notably, we first consider ``DICE'' which involves iteratively perturbing a graph's structure by adding or removing edges while ensuring connectivity, and then adjusting the perturbation based on the gradient of the graph neural network's loss function to generate an adversarial example. The process aims to find a minimal perturbation that misleads the network's predictions while keeping the perturbation size small. We additionally consider a ``Random'' attack which consists of randomly perturbing the adjacency matrix by dropping or adding edges. Figure \ref{fig:dice_gaussian} shows the adversarial accuracy results on the Cora and CiteSeer dataset when subject to DICE and Random attacks for different values of $\sigma$ of the Gaussian initialization. Similarly, Figure \ref{fig:dice_uniform} shows the effect of scaling both a uniform initialization and an Orthogonal one as previously explained in Section \ref{sec:experimental_results}. 

\begin{figure}[h]
    \centering
    \includegraphics[width=1.0\textwidth]{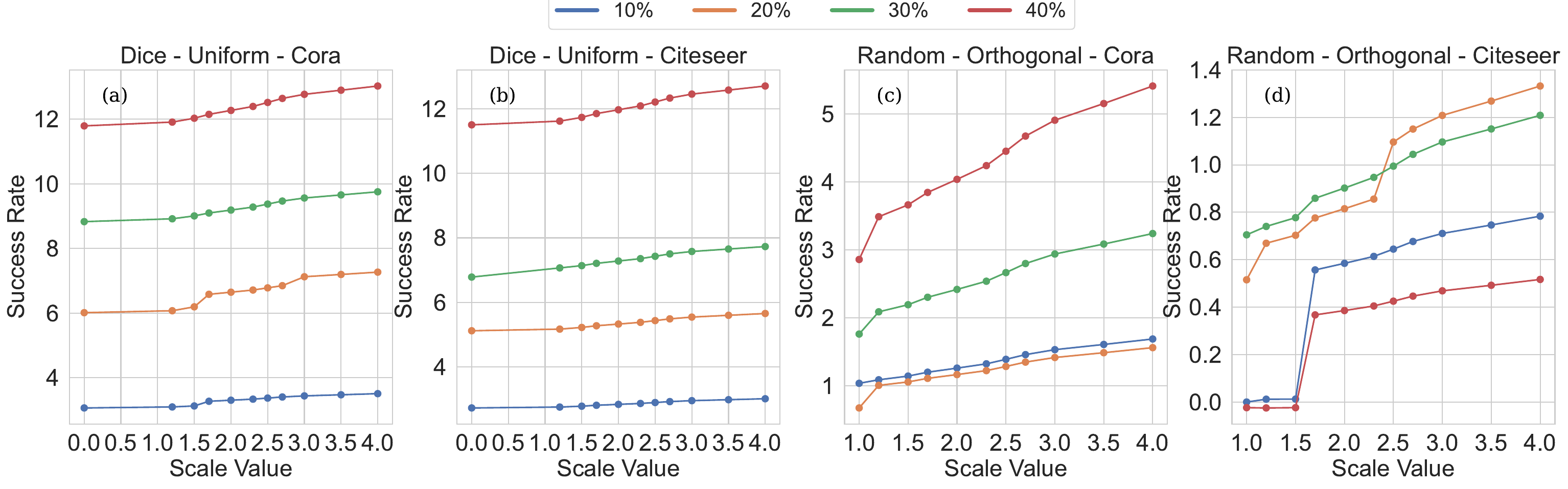}
    \caption{Effect of Uniform and Orthogonal Initialization on the model's robustness in the case of DICE Attack on Cora (a,c) and CiteSeer (b,d).}
    \label{fig:dice_uniform}
\end{figure}

\begin{figure}[h]
    \centering
    \includegraphics[width=0.7\textwidth]{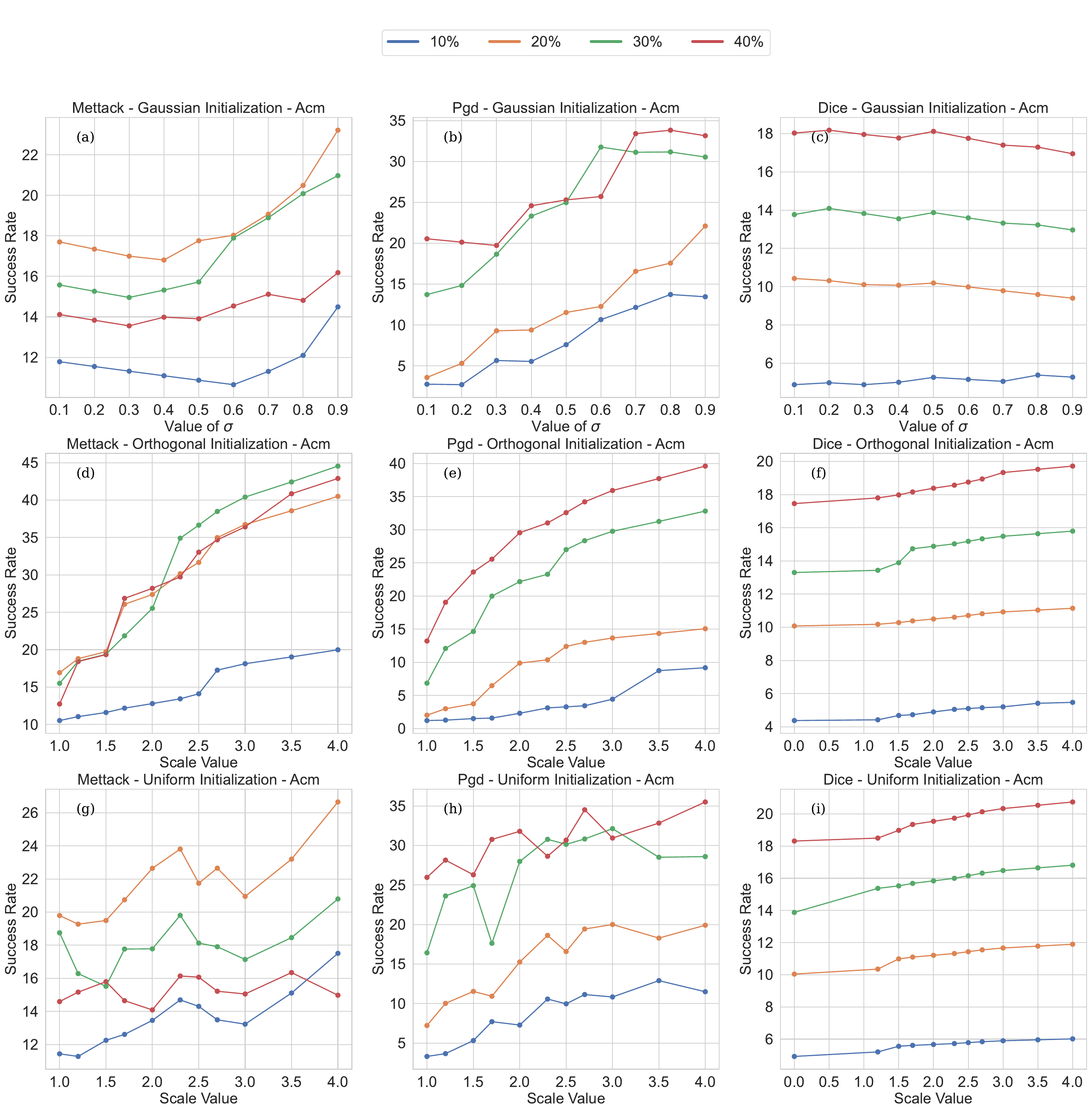}
    \caption{Effect of the Gaussian (a; b; c), Orthogonal (d; e; f) and Uniform (g;h;i) Initialization on the ACM dataset.}
    \label{fig:result_acm}
\end{figure}

\subsection{Additional Datasets}

We additionally extend the results to the ACM Dataset \cite{wang2019heterogeneous} within the node classification setting. Figure \ref{fig:result_acm} presents the results using the Mettack, PGD and DICE for the ACM dataset for the Gaussian initialization (effect of $\sigma$), the Uniform and Orthogonal initialization.

\begin{figure}[h]
    \centering
    \includegraphics[width=0.7\textwidth]{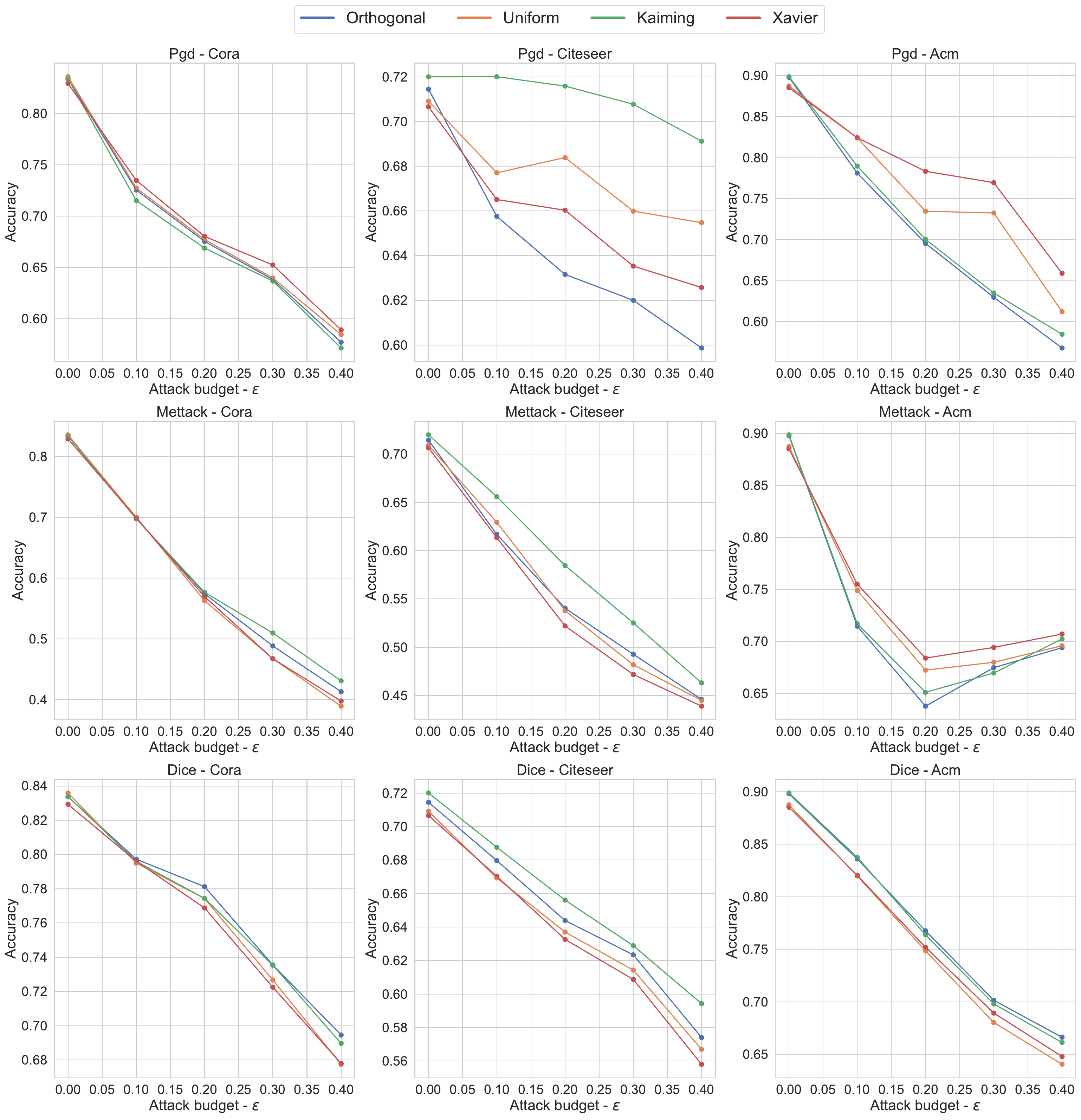}
    \caption{Effect of the initial distribution on RGCN's robustness and performance when subject to structural adversarial attacks.}
    \label{fig:result_rgcn}
\end{figure}

\subsection{Additional Models}

As previously explained in Section \ref{sec:generalization}, while our theoretical analysis primarily focuses on GCN, GIN, and DNN models, the derived insights extend to other models as well. To illustrate this point, we examine the effect of initialization distribution on the performance of defense methodologies. Specifically, we first consider RGCN \cite{zhu2019robust}, which employs Gaussian distributions in its hidden layers to mitigate the effects of adversarial attacks. We additionally consider GCN-Jaccard \cite{gnn_jaccard} which preprocesses the network by eliminating edges that connect nodes with jaccard similarity of features smaller than a certain level. We use various initialization schemes, similar to those in our previous experiments, and evaluate against the same adversarial attacks (PGD, Mettack, and DICE). Figure \ref{fig:result_rgcn} (resp. Figure \ref{fig:result_jaccard}) presents the adversarial accuracy and defense performance of RGCN (resp. GCN-Jaccard) on the Cora, CiteSeer, and ACM datasets. Although the performance gap is not very pronounced for Cora, it is clearly observed for CiteSeer and ACM. This demonstrates the broader applicability of our insights across different models but also defense methods.

\begin{figure}[hbt!]
    \centering
    \includegraphics[width=0.7\textwidth]{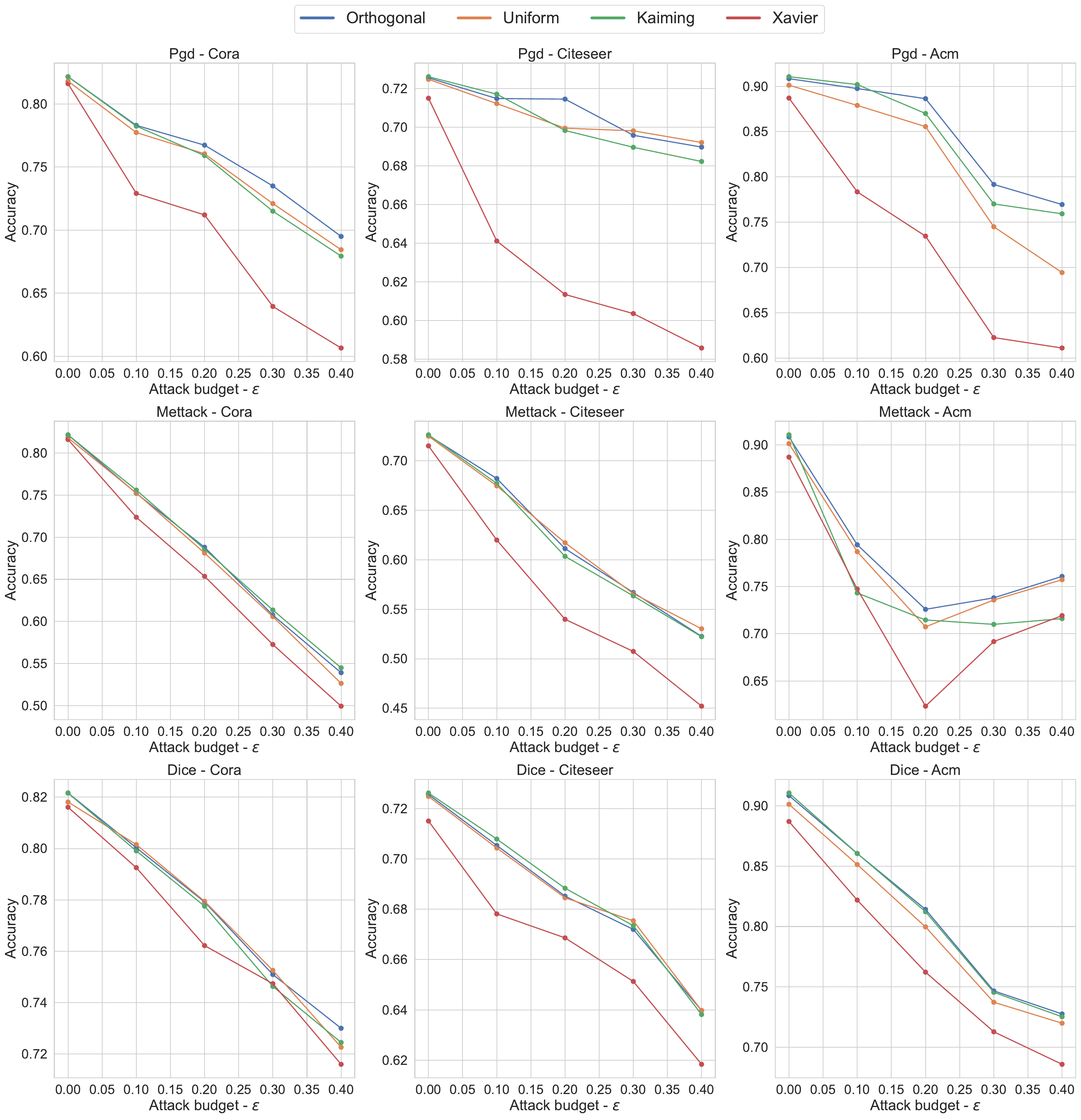}
    \caption{Effect of the initial distribution on GCN-Jaccard's robustness and performance when subject to structural adversarial attacks.}
    \label{fig:result_jaccard}
\end{figure}

\section{Datasets and Implementation details} \label{appendix:dataset_implementation_details}

\textbf{Datasets} Characteristics and information about the node classification datasets used in our experimental study are presented in Table \ref{tab:data_statistics}. As outlined in the main paper, we conduct experiments on a set of citation networks, including Cora, CiteSeer (in the main paper), and ACM dataset (Appendix \ref{appendix:additional_results}) \cite{wang2019heterogeneous}. For all these datasets, we adhere to the train/valid/test splits provided by with the dataset. 

\textbf{About the architectures.} In all of the experiments, the models employed a 2-layer convolutional architecture (consisting of two iterations of message passing and updating) stacked with a Multi-Layer Perception (MLP) as a readout. The intent was to compare the models in an iso-architectural setting, to ensure a fair evaluation of their robustness. We maintained the same hyperparameters, including a learning rate of 1e-2, 300 epochs, and a hidden feature dimension of 16 have been. To account for the impact of random initialization, each experiment was repeated 10 times.

\textbf{Reproducibility of the experiments.} We emphasize that all experiments should be easily reproducible by directly using the provided code. The archive contains a ReadMe file containing a small documentation on how to run the experiments. 

\begin{table}[h]
\caption{Statistics of the node classification datasets used in our experiments.}
\label{tab:data_statistics}
\vskip 0.15in
\begin{center}
\begin{small}
\begin{sc}
\begin{tabular}{lcccc}
\toprule
Dataset & \#Features & \#Nodes & \#Edges & \#Classes \\
\midrule
Cora    & 1433 & 2708 & 5208 & 7 \\
CiteSeer    & 3703 & 3327 & 4552 & 6 \\
\bottomrule
\end{tabular}
\end{sc}
\end{small}
\end{center}
\vskip -0.1in
\end{table}


\textbf{On the adversarial attacks.} For the PGD attack on the MNIST dataset, we used a step-size of $0.1$ and we set the number of iterations to $100$ (which was observed to be enough for the attack convergence). Note that we set these parameters for all the considered initializations in Figure \ref{fig:other_inits} as our aim is to compare the effect of the different distribution on the final robustness.

\textbf{Implementation details.} Our implementation is available in the supplementary materials (and will be publicly available afterwards). It is built using the open-source library \textit{PyTorch Geometric} (PyG) under the MIT license \citep{Fey/Lenssen/2019}. We used the publicly available implementation of the adversarial attacks provided in the DeepRobust package (https://github.com/DSE-MSU/DeepRobust). For RGCN, we used the implementation from the same package. The experiments have been run on both a NVIDIA A100 GPU where training a GCN takes around $1.2(\pm0.2)$ s. 


\newpage

\section*{NeurIPS Paper Checklist}

\begin{enumerate}

\item {\bf Claims}
    \item[] Question: Do the main claims made in the abstract and introduction accurately reflect the paper's contributions and scope?
    \item[] Answer: \answerYes{} 
    \item[] Justification: In addition to stating the novelty of our proposed approach, we used our abstract and introduction to summarize our main findings and contributions related to the effect of initialization on the adversarial robustness (as theoretically justified and empirically tested in the following sections).
    \item[] Guidelines:
    \begin{itemize}
        \item The answer NA means that the abstract and introduction do not include the claims made in the paper.
        \item The abstract and/or introduction should clearly state the claims made, including the contributions made in the paper and important assumptions and limitations. A No or NA answer to this question will not be perceived well by the reviewers. 
        \item The claims made should match theoretical and experimental results, and reflect how much the results can be expected to generalize to other settings. 
        \item It is fine to include aspirational goals as motivation as long as it is clear that these goals are not attained by the paper. 
    \end{itemize}

\item {\bf Limitations}
    \item[] Question: Does the paper discuss the limitations of the work performed by the authors?
    \item[] Answer: \answerYes{} 
    \item[] Justification: Together with our conclusion, we presented the set of limitations of work. Specifically, we stated that while our work is innovative, we didn't provide a solution to the initialization problem from an adversarial defense perspective. We also discussed in the "problem setup” section our different theoretical choices (the smoothness of the loss function) and how realistic they are. 
    \item[] Guidelines:
    \begin{itemize}
        \item The answer NA means that the paper has no limitation while the answer No means that the paper has limitations, but those are not discussed in the paper. 
        \item The authors are encouraged to create a separate "Limitations" section in their paper.
        \item The paper should point out any strong assumptions and how robust the results are to violations of these assumptions (e.g., independence assumptions, noiseless settings, model well-specification, asymptotic approximations only holding locally). The authors should reflect on how these assumptions might be violated in practice and what the implications would be.
        \item The authors should reflect on the scope of the claims made, e.g., if the approach was only tested on a few datasets or with a few runs. In general, empirical results often depend on implicit assumptions, which should be articulated.
        \item The authors should reflect on the factors that influence the performance of the approach. For example, a facial recognition algorithm may perform poorly when image resolution is low or images are taken in low lighting. Or a speech-to-text system might not be used reliably to provide closed captions for online lectures because it fails to handle technical jargon.
        \item The authors should discuss the computational efficiency of the proposed algorithms and how they scale with dataset size.
        \item If applicable, the authors should discuss possible limitations of their approach to address problems of privacy and fairness.
        \item While the authors might fear that complete honesty about limitations might be used by reviewers as grounds for rejection, a worse outcome might be that reviewers discover limitations that aren't acknowledged in the paper. The authors should use their best judgment and recognize that individual actions in favor of transparency play an important role in developing norms that preserve the integrity of the community. Reviewers will be specifically instructed to not penalize honesty concerning limitations.
    \end{itemize}

\item {\bf Theory Assumptions and Proofs}
    \item[] Question: For each theoretical result, does the paper provide the full set of assumptions and a complete (and correct) proof?
    \item[] Answer: \answerYes{} 
    \item[] Justification: For each Theorem, Lemma and theoretical claim, we provide the proof in the Appendix and point out to the corresponding section in the main paper. We also stated all the assumptions and analytical choices in the Preliminaries (Section \ref{sec:preliminaries})
    \item[] Guidelines:
    \begin{itemize}
        \item The answer NA means that the paper does not include theoretical results. 
        \item All the theorems, formulas, and proofs in the paper should be numbered and cross-referenced.
        \item All assumptions should be clearly stated or referenced in the statement of any theorems.
        \item The proofs can either appear in the main paper or the supplemental material, but if they appear in the supplemental material, the authors are encouraged to provide a short proof sketch to provide intuition. 
        \item Inversely, any informal proof provided in the core of the paper should be complemented by formal proofs provided in appendix or supplemental material.
        \item Theorems and Lemmas that the proof relies upon should be properly referenced. 
    \end{itemize}

    \item {\bf Experimental Result Reproducibility}
    \item[] Question: Does the paper fully disclose all the information needed to reproduce the main experimental results of the paper to the extent that it affects the main claims and/or conclusions of the paper (regardless of whether the code and data are provided or not)?
    \item[] Answer: \answerYes{} 
    \item[] Justification: In addition to providing the code as supplementary materials, we have provided all the implementations details that are sufficient to reproduce the results. These details include the used hyper-parameters (the architecture, learning rate \ldots) and also for the used adversarial attacks we provide the different parameters used. We also point out the dataset that we used (which are public) and that we used the same public folds as the one provided with the datasets. 
    \item[] Guidelines:
    \begin{itemize}
        \item The answer NA means that the paper does not include experiments.
        \item If the paper includes experiments, a No answer to this question will not be perceived well by the reviewers: Making the paper reproducible is important, regardless of whether the code and data are provided or not.
        \item If the contribution is a dataset and/or model, the authors should describe the steps taken to make their results reproducible or verifiable. 
        \item Depending on the contribution, reproducibility can be accomplished in various ways. For example, if the contribution is a novel architecture, describing the architecture fully might suffice, or if the contribution is a specific model and empirical evaluation, it may be necessary to either make it possible for others to replicate the model with the same dataset, or provide access to the model. In general. releasing code and data is often one good way to accomplish this, but reproducibility can also be provided via detailed instructions for how to replicate the results, access to a hosted model (e.g., in the case of a large language model), releasing of a model checkpoint, or other means that are appropriate to the research performed.
        \item While NeurIPS does not require releasing code, the conference does require all submissions to provide some reasonable avenue for reproducibility, which may depend on the nature of the contribution. For example
        \begin{enumerate}
            \item If the contribution is primarily a new algorithm, the paper should make it clear how to reproduce that algorithm.
            \item If the contribution is primarily a new model architecture, the paper should describe the architecture clearly and fully.
            \item If the contribution is a new model (e.g., a large language model), then there should either be a way to access this model for reproducing the results or a way to reproduce the model (e.g., with an open-source dataset or instructions for how to construct the dataset).
            \item We recognize that reproducibility may be tricky in some cases, in which case authors are welcome to describe the particular way they provide for reproducibility. In the case of closed-source models, it may be that access to the model is limited in some way (e.g., to registered users), but it should be possible for other researchers to have some path to reproducing or verifying the results.
        \end{enumerate}
    \end{itemize}

\item {\bf Open access to data and code}
    \item[] Question: Does the paper provide open access to the data and code, with sufficient instructions to faithfully reproduce the main experimental results, as described in supplemental material?
    \item[] Answer: \answerYes{} 
    \item[] Justification: We provide the anonymized code following the Neurips guidelines. Specifically, we submitted the code with the supplementary material section and we clearly state the steps to run it using a ReadMe file. Please note that for this question, we consider "open source" as providing the code to the reviewers and making it public afterwards for the public. 
    \item[] Guidelines:
    \begin{itemize}
        \item The answer NA means that paper does not include experiments requiring code.
        \item Please see the NeurIPS code and data submission guidelines (\url{https://nips.cc/public/guides/CodeSubmissionPolicy}) for more details.
        \item While we encourage the release of code and data, we understand that this might not be possible, so “No” is an acceptable answer. Papers cannot be rejected simply for not including code, unless this is central to the contribution (e.g., for a new open-source benchmark).
        \item The instructions should contain the exact command and environment needed to run to reproduce the results. See the NeurIPS code and data submission guidelines (\url{https://nips.cc/public/guides/CodeSubmissionPolicy}) for more details.
        \item The authors should provide instructions on data access and preparation, including how to access the raw data, preprocessed data, intermediate data, and generated data, etc.
        \item The authors should provide scripts to reproduce all experimental results for the new proposed method and baselines. If only a subset of experiments are reproducible, they should state which ones are omitted from the script and why.
        \item At submission time, to preserve anonymity, the authors should release anonymized versions (if applicable).
        \item Providing as much information as possible in supplemental material (appended to the paper) is recommended, but including URLs to data and code is permitted.
    \end{itemize}

\item {\bf Experimental Setting/Details}
    \item[] Question: Does the paper specify all the training and test details (e.g., data splits, hyperparameters, how they were chosen, type of optimizer, etc.) necessary to understand the results?
    \item[] Answer: \answerYes{} 
    \item[] Justification: We provided all the details about the architecture, the used hyper-parameters for the considered models (Section H of the Appendix) and all the hyper-parameters used for our adversarial attacks. Note that our work's goal is to provide comprehensive overview of the effect of initialization on the robustness, hence making sure that the same choice of hyper-parameters is enough to ensure the fairness of the experiments. 
    \item[] Guidelines:
    \begin{itemize}
        \item The answer NA means that the paper does not include experiments.
        \item The experimental setting should be presented in the core of the paper to a level of detail that is necessary to appreciate the results and make sense of them.
        \item The full details can be provided either with the code, in appendix, or as supplemental material.
    \end{itemize}

\item {\bf Experiment Statistical Significance}
    \item[] Question: Does the paper report error bars suitably and correctly defined or other appropriate information about the statistical significance of the experiments?
    \item[] Answer:  \answerNo{} 
    \item[] Justification: We reproduce each experiment 10 times to take into account the factor of randomization and we report the mean value. Note that since we use mainly figures (which are appropriate for our setting -- given the different attack budgets we are using), this seemed as the perfect approach. For the train/test folds, we use the public folds provided with each dataset and hence reducing the effect of randomization. 
    \item[] Guidelines:
    \begin{itemize}
        \item The answer NA means that the paper does not include experiments.
        \item The authors should answer "Yes" if the results are accompanied by error bars, confidence intervals, or statistical significance tests, at least for the experiments that support the main claims of the paper.
        \item The factors of variability that the error bars are capturing should be clearly stated (for example, train/test split, initialization, random drawing of some parameter, or overall run with given experimental conditions).
        \item The method for calculating the error bars should be explained (closed form formula, call to a library function, bootstrap, etc.)
        \item The assumptions made should be given (e.g., Normally distributed errors).
        \item It should be clear whether the error bar is the standard deviation or the standard error of the mean.
        \item It is OK to report 1-sigma error bars, but one should state it. The authors should preferably report a 2-sigma error bar than state that they have a 96\% CI, if the hypothesis of Normality of errors is not verified.
        \item For asymmetric distributions, the authors should be careful not to show in tables or figures symmetric error bars that would yield results that are out of range (e.g. negative error rates).
        \item If error bars are reported in tables or plots, The authors should explain in the text how they were calculated and reference the corresponding figures or tables in the text.
    \end{itemize}

\item {\bf Experiments Compute Resources}
    \item[] Question: For each experiment, does the paper provide sufficient information on the computer resources (type of compute workers, memory, time of execution) needed to reproduce the experiments?
    \item[] Answer: \answerYes{} 
    \item[] Justification: We reported the details of implementation in Section H of the Appendix, where we specified the GPU that was used and the average time to do the experiments. Note that while we have chosen to use a GPU, our experiments can be easily done using a CPU.
    \item[] Guidelines:
    \begin{itemize}
        \item The answer NA means that the paper does not include experiments.
        \item The paper should indicate the type of compute workers CPU or GPU, internal cluster, or cloud provider, including relevant memory and storage.
        \item The paper should provide the amount of compute required for each of the individual experimental runs as well as estimate the total compute. 
        \item The paper should disclose whether the full research project required more compute than the experiments reported in the paper (e.g., preliminary or failed experiments that didn't make it into the paper). 
    \end{itemize}
    
\item {\bf Code Of Ethics}
    \item[] Question: Does the research conducted in the paper conform, in every respect, with the NeurIPS Code of Ethics \url{https://neurips.cc/public/EthicsGuidelines}?
    \item[] Answer: \answerYes{} 
    \item[] Justification: We follow the guidelines of the Neurips Code of Ethics.
    \item[] Guidelines:
    \begin{itemize}
        \item The answer NA means that the authors have not reviewed the NeurIPS Code of Ethics.
        \item If the authors answer No, they should explain the special circumstances that require a deviation from the Code of Ethics.
        \item The authors should make sure to preserve anonymity (e.g., if there is a special consideration due to laws or regulations in their jurisdiction).
    \end{itemize}

\item {\bf Broader Impacts}
    \item[] Question: Does the paper discuss both potential positive societal impacts and negative societal impacts of the work performed?
    \item[] Answer: \answerYes{} 
    \item[] Justification: We provided overview on the harm that adversarial attacks can have on the applications of Deep Learning models. The main goal of our paper is to identify new potential factors related to adversarial attacks and hence should rather have a positive impact on the society.
    \item[] Guidelines:
    \begin{itemize}
        \item The answer NA means that there is no societal impact of the work performed.
        \item If the authors answer NA or No, they should explain why their work has no societal impact or why the paper does not address societal impact.
        \item Examples of negative societal impacts include potential malicious or unintended uses (e.g., disinformation, generating fake profiles, surveillance), fairness considerations (e.g., deployment of technologies that could make decisions that unfairly impact specific groups), privacy considerations, and security considerations.
        \item The conference expects that many papers will be foundational research and not tied to particular applications, let alone deployments. However, if there is a direct path to any negative applications, the authors should point it out. For example, it is legitimate to point out that an improvement in the quality of generative models could be used to generate deepfakes for disinformation. On the other hand, it is not needed to point out that a generic algorithm for optimizing neural networks could enable people to train models that generate Deepfakes faster.
        \item The authors should consider possible harms that could arise when the technology is being used as intended and functioning correctly, harms that could arise when the technology is being used as intended but gives incorrect results, and harms following from (intentional or unintentional) misuse of the technology.
        \item If there are negative societal impacts, the authors could also discuss possible mitigation strategies (e.g., gated release of models, providing defenses in addition to attacks, mechanisms for monitoring misuse, mechanisms to monitor how a system learns from feedback over time, improving the efficiency and accessibility of ML).
    \end{itemize}
    
\item {\bf Safeguards}
    \item[] Question: Does the paper describe safeguards that have been put in place for responsible release of data or models that have a high risk for misuse (e.g., pretrained language models, image generators, or scraped datasets)?
    \item[] Answer: \answerNA{}{} 
    \item[] Justification: In this work, we study the theoretical effect of initialization on the adversarial robustness. We don't provide any new pre-trained model nor new datasets.
    \item[] Guidelines:
    \begin{itemize}
        \item The answer NA means that the paper poses no such risks.
        \item Released models that have a high risk for misuse or dual-use should be released with necessary safeguards to allow for controlled use of the model, for example by requiring that users adhere to usage guidelines or restrictions to access the model or implementing safety filters. 
        \item Datasets that have been scraped from the Internet could pose safety risks. The authors should describe how they avoided releasing unsafe images.
        \item We recognize that providing effective safeguards is challenging, and many papers do not require this, but we encourage authors to take this into account and make a best faith effort.
    \end{itemize}

\item {\bf Licenses for existing assets}
    \item[] Question: Are the creators or original owners of assets (e.g., code, data, models), used in the paper, properly credited and are the license and terms of use explicitly mentioned and properly respected?
    \item[] Answer: \answerYes{} 
    \item[] Justification: We made sure to cite the papers that are relevant to our work and that were used to justify some theoretical or empirical insights. For the different code implementations, we cited clearly the license and the owner of the used function/code. 
    \item[] Guidelines:
    \begin{itemize}
        \item The answer NA means that the paper does not use existing assets.
        \item The authors should cite the original paper that produced the code package or dataset.
        \item The authors should state which version of the asset is used and, if possible, include a URL.
        \item The name of the license (e.g., CC-BY 4.0) should be included for each asset.
        \item For scraped data from a particular source (e.g., website), the copyright and terms of service of that source should be provided.
        \item If assets are released, the license, copyright information, and terms of use in the package should be provided. For popular datasets, \url{paperswithcode.com/datasets} has curated licenses for some datasets. Their licensing guide can help determine the license of a dataset.
        \item For existing datasets that are re-packaged, both the original license and the license of the derived asset (if it has changed) should be provided.
        \item If this information is not available online, the authors are encouraged to reach out to the asset's creators.
    \end{itemize}

\item {\bf New Assets}
    \item[] Question: Are new assets introduced in the paper well documented and is the documentation provided alongside the assets?
    \item[] Answer: \answerYes{} 
    \item[] Justification: We have provided the implementation code together with all the experimental details to reproduce our work. We also clearly justify the use of the packages and their license. Note that the code have been anonymized and provided as a supplementary materials.
    \item[] Guidelines:
    \begin{itemize}
        \item The answer NA means that the paper does not release new assets.
        \item Researchers should communicate the details of the dataset/code/model as part of their submissions via structured templates. This includes details about training, license, limitations, etc. 
        \item The paper should discuss whether and how consent was obtained from people whose asset is used.
        \item At submission time, remember to anonymize your assets (if applicable). You can either create an anonymized URL or include an anonymized zip file.
    \end{itemize}

\item {\bf Crowdsourcing and Research with Human Subjects}
    \item[] Question: For crowdsourcing experiments and research with human subjects, does the paper include the full text of instructions given to participants and screenshots, if applicable, as well as details about compensation (if any)? 
    \item[] Answer: \answerNA{} 
    \item[] Justification: There is no crowdsourcing nor research with human subjects in our case. 
    \item[] Guidelines:
    \begin{itemize}
        \item The answer NA means that the paper does not involve crowdsourcing nor research with human subjects.
        \item Including this information in the supplemental material is fine, but if the main contribution of the paper involves human subjects, then as much detail as possible should be included in the main paper. 
        \item According to the NeurIPS Code of Ethics, workers involved in data collection, curation, or other labor should be paid at least the minimum wage in the country of the data collector. 
    \end{itemize}

\item {\bf Institutional Review Board (IRB) Approvals or Equivalent for Research with Human Subjects}
    \item[] Question: Does the paper describe potential risks incurred by study participants, whether such risks were disclosed to the subjects, and whether Institutional Review Board (IRB) approvals (or an equivalent approval/review based on the requirements of your country or institution) were obtained?
    \item[] Answer: \answerNA{} 
    \item[] Justification: There is no crowdsourcing nor research with human subjects in our case. 
    \item[] Guidelines:
    \begin{itemize}
        \item The answer NA means that the paper does not involve crowdsourcing nor research with human subjects.
        \item Depending on the country in which research is conducted, IRB approval (or equivalent) may be required for any human subjects research. If you obtained IRB approval, you should clearly state this in the paper. 
        \item We recognize that the procedures for this may vary significantly between institutions and locations, and we expect authors to adhere to the NeurIPS Code of Ethics and the guidelines for their institution. 
        \item For initial submissions, do not include any information that would break anonymity (if applicable), such as the institution conducting the review.
    \end{itemize}

\end{enumerate}

\end{document}